\documentclass[11pt, a4paper, twocolumn]{google}

\PassOptionsToPackage{numbers,compress}{natbib}

\usepackage[utf8]{inputenc} %
\usepackage[T1]{fontenc}    %
\usepackage{amsfonts}       %

\usepackage{natbib}
\usepackage{listings}
\usepackage{subfig}
\usepackage{comment}
\usepackage{algorithm}
\usepackage{algpseudocode}

\usepackage{geometry}
\usepackage{tikz}
\usetikzlibrary{shapes.geometric, arrows.meta, positioning, fit,trees}

\usepackage{varwidth}
\usepackage{enumitem}
\usepackage{placeins}
\usepackage{pgfplotstable}
\usepackage{listings}

\usepackage{amsmath,amsfonts,bm}
\usepackage{amsthm}

\def\eqref#1{equation~\ref{#1}}

\def\1{\bm{1}}

\newcommand{\valid}{\mathcal{D_{\mathrm{valid}}}}

\def\vtheta{{\bm{\theta}}}
\def\vTheta{{\bm{\Theta}}}
\def\vphi{{\bm{\phi}}}

\def\vx{{\bm{x}}}
\def\vy{{\bm{y}}}

\DeclareMathAlphabet{\mathsfit}{\encodingdefault}{\sfdefault}{m}{sl}
\SetMathAlphabet{\mathsfit}{bold}{\encodingdefault}{\sfdefault}{bx}{n}

\DeclareMathOperator*{\argmax}{arg\,max}

\usepackage{todonotes}
\makeatletter
\newcommand*\iftodonotes{\if@todonotes@disabled\expandafter\@secondoftwo\else\expandafter\@firstoftwo\fi}  %
\makeatother

\tikzstyle{latent} = [circle, minimum width=1.5cm, draw=blue, fill=blue!20, thick, inner sep=0pt, dashed]
\tikzstyle{observed} = [circle, minimum width=1.5cm, draw=black, fill=black!20, thick, inner sep=0pt]

\newcommand{\hyperedge}[5][180]{
     \draw (#2.#1) ++(#1:1.2)  edge (#2) edge (#3) edge[->] node[above] {#5} (#4);
}

\newcommand{\hyperedgeleft}[5][180]{
     \draw (#2.#1) ++(#1:1.2)  edge (#2) edge (#3) edge[->] node[left] {#5} (#4);
}

\usepackage{listings}
\usepackage{siunitx}

\usepackage{booktabs}
\usepackage{hyperref}
\usepackage{url}
\usepackage[export]{adjustbox}

\usepackage[noabbrev,capitalize]{cleveref} %
\usepackage{nicefrac}       %
\usepackage[activate=true,
    final,
    tracking=false,   %
    kerning=true,
    spacing=true,
    factor=1050,
    stretch=30,
    shrink=30]{microtype}      %
\microtypecontext{spacing=nonfrench}
\newcommand{\model}{\textsc{tac}}
\newtheorem{lemma}{Lemma}
\usepackage{thm-restate}
\DeclareRobustCommand{\thinskip}{\hskip 0.16667em\relax}
\def\emdash{---}
\def\d@sh#1#2{\unskip#1\thinskip#2\thinskip\ignorespaces}
\def\Dash{\d@sh\nobreak\emdash}
\def\Ldash{\d@sh\empty{\hbox{\emdash}\nobreak}}
\def\Rdash{\d@sh\nobreak\emdash}
\usepackage{inconsolata}

\usepackage{graphicx}

\author{Chu-Cheng Lin}
\author{Daiyi Peng}
\author{Yifeng Lu}
\author{Ming Zhang}
\author{Eugene Ie}
\uselogo{}

\affil{Google}
\correspondingauthor{kitsing@google.com}

\usepackage{xcolor}         %

\title{Type-Compliant Adaptation Cascades: Adapting Programmatic LM Workflows to Data}

\begin{abstract}
Reliably composing Large Language Models (LLMs) for complex, multi-step workflows remains a significant challenge. The dominant paradigm \Dash optimizing discrete prompts in a pipeline \Dash is notoriously brittle and struggles to enforce the formal compliance required for structured tasks. We introduce Type-Compliant Adaptation Cascades (\model s), a framework that recasts workflow adaptation as learning typed probabilistic programs. \model s treat the entire workflow, which is composed of parameter-efficiently adapted LLMs and deterministic logic, as an unnormalized joint distribution. This enables principled, gradient-based training even with latent intermediate structures. We provide theoretical justification for our tractable optimization objective, proving that the optimization bias vanishes as the model learns type compliance. Empirically, \model s significantly outperform state-of-the-art prompt-optimization baselines. Gains are particularly pronounced on structured tasks, improving FinQA from $12.0\%$ to $24.7\%$ for a Qwen 3 8B model, MGSM-SymPy from $57.1\%$ to $75.9\%$ for a Gemma 2 27B model, MGSM from $1.6\%$ to $27.3\%$, and MuSR from $36.5\%$ to $62.6\%$ for a Gemma 7B model. \model s offer a robust and theoretically grounded paradigm for developing reliable, task-compliant LLM systems.

\end{abstract}

\begin{document}

\maketitle

\section{Introduction}
\label{sec:intro}

\begin{figure*}[ht!]
    \centering
    \subfloat[{\bf cot-cascade-structure}]{\label{fig:cot-cascade-structure}
    \resizebox{.2\textwidth}{!}{
    \begin{tikzpicture}[>=Stealth]
    \node[observed, align=center] (x) {$\mathbf{z}_1$ \\ type: $\tau_i$};
    \node[latent, align=center, right=1.5cm of x] (r) {$\mathbf{z}_3$ \\ type: $\tau_r$};
    \node[latent, align=center,  below left=2cm and .5cm of r] (ir) {$\mathbf{z}_4$ \\ type: $\tau_{ir}$};
    \node[observed, align=center, below=2cm of ir] (y) {$\mathbf{z}_2$ \\ type: $\tau_o$};

    \draw[->] (x) edge node[above] {$(\tau_i, \tau_{r}, \vtheta_3)$}  (r);
    \hyperedge[-60]{x}{r}{ir}{$\mathrm{combine\_ir}: \tau_{i} \times \tau_r \rightarrow \tau_{ir}$};
    \draw[->] (ir) edge node[above] {$(\tau_{ir}, \tau_{o}, \vtheta_4)$}  (y);
\end{tikzpicture}}}%
\subfloat[{\bf expression-cascade-structure}]{
  \includegraphics[width=.3\linewidth]{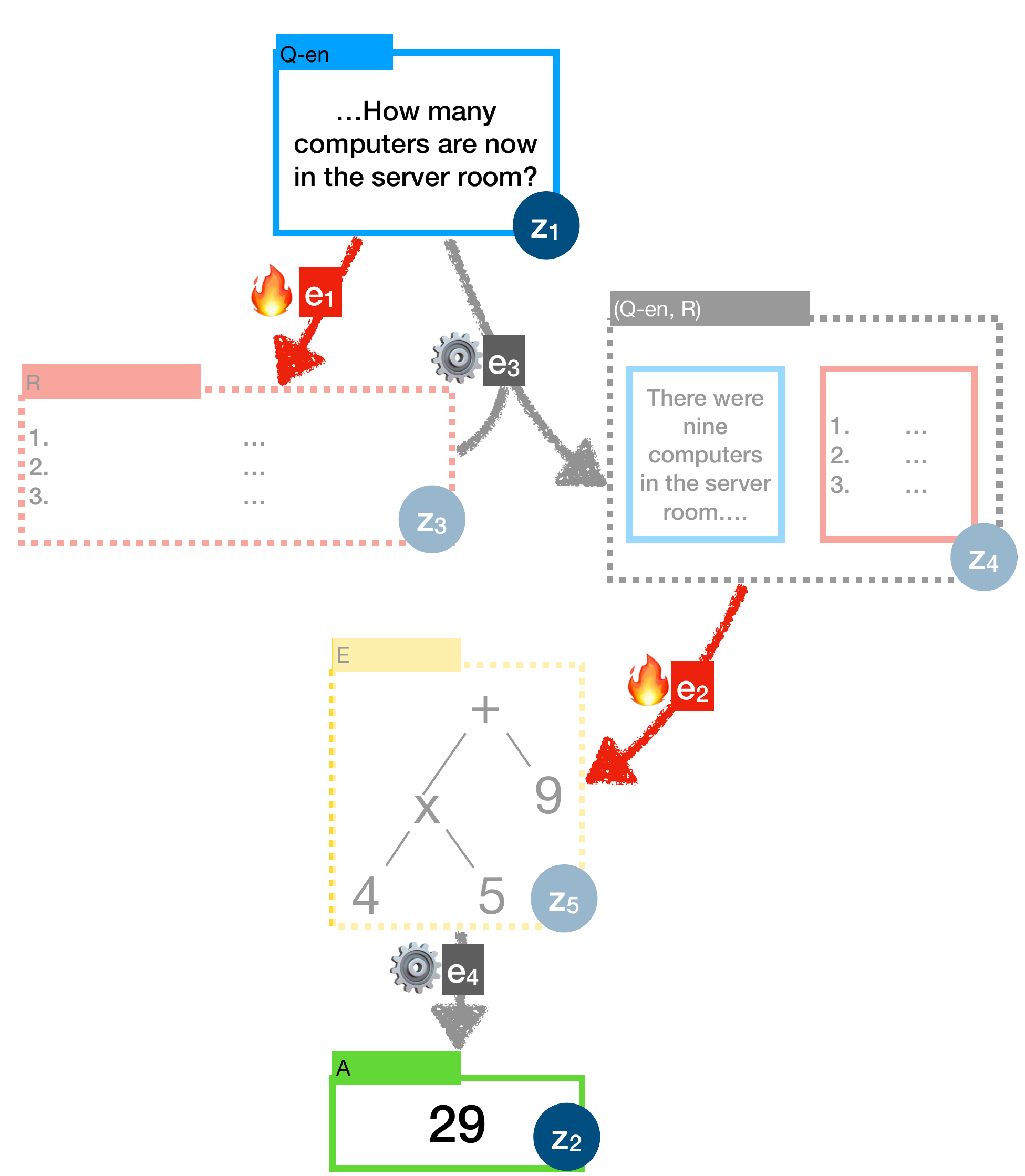}
    \label{fig:cascade}
}
\caption{\small
    Two \model\,workflow patterns experimented in this paper. We illustrate the more complicated \cref{fig:cascade} with example node values (we also explore additional patterns in \cref{sec:additional-studies-pattern-design}). Dashed-boundary nodes indicate variables whose values are not available in annotated data, and solid-boundary nodes indicate nodes with training time observable values. 
    A main message of this work is that {\bf we can treat an entire typed workflow as a single probabilistic program, whose parameters are lightweight PEFT modules, allowing end-to-end training with latent variables}, instead of defining workflows imperatively as fixed-parameter systems. }
    \label{fig:tac-patterns}
\end{figure*}
{Language modeling} \citep{Rosenfeld2018} refers to fitting a parametric probability distribution over strings (a \emph{language model}) $p_{\vtheta}$ to observed data.
Large Language Models (LLMs) \citep{gpt3} scale both the model and training datasets to massive sizes. LLMs have an extraordinary emergent capability: once trained, these distributions can be effectively manipulated simply by \emph{asking} \Dash conditioning the distribution on different natural language instruction prefixes \citep{wei2022finetuned} \Dash a practice widely known as prompting.

The expressive power and accessibility of this natural language interface have catalyzed the rapid development of programmatically composed workflows and agentic systems \citep{khattab2022demonstrate, Chase_LangChain_2022,yao2023react,wu2024autogen}. By structuring inputs and chaining model calls, practitioners can construct complex systems capable of multi-step reasoning and interaction. However, the success of these systems is inherently subject to the pretrained LLM’s capabilities in instruction following. Moreover, prompt engineering remains brittle: minor textual variations can lead to drastic performance degradation \citep{cao2024on}.
This brittleness can also cause type violations in a programmatic workflow: while inference-time constrained decoding methods mitigate type violation problems, full compliance remains theoretically impossible for complex types \citep{lin-etal-2021-limitations} on autoregressive models.
Optimizing these composed systems therefore often devolves into a difficult discrete optimization problem over the space of possible prompts \Dash a challenge often addressed through heuristic search \citep{zhou2023large,pryzant-etal-2023-automatic,yuksekgonul2025optimizing} and reinforcement learning \citep{jafari-etal-2024-morl}, both of which suffer from variance issues.

In this paper, we propose a return to the foundational perspective: {fitting} composed LLM distributions to downstream tasks as \emph{parametric probability models}. Instead of tackling the inherent difficulties of optimizing discrete verbal instructions, we \emph{adapt} a composed workflow (such as ones shown in \cref{fig:tac-patterns}), as a parametric latent variable model, to maximize data likelihood.
Each step in the workflow is a probabilistic typed transformation backed by a parameter-efficient fine-tuning (PEFT) adaptor, with valid typed objects as its support.
Different workflows are declaratively defined as different generative stories that sequentially transform objects with either learned adaptors or deterministic algorithms.
Thus, we transform the problem of workflow adaptation from an ad-hoc, discrete optimization search problem to training and inference of latent variable models. This allows us to leverage well-established machine learning techniques to optimize the entire system directly, while keeping training and inference manageable, thanks to the adaptors' parameter and computational efficiency.

This approach, which we term Type-Compliant Adaptation Cascades (\model s), is an end-to-end trainable probabilistic programming framework.
As parametric latent variable models, \model s can be optimized using gradient descent methods. Moreover, as unnormalized distributions over typed objects, Posterior inference of \model s is decoupled from training, enabling techniques such as amortized inference and classification by ranking.

Our primary contributions are:
\begin{itemize}[leftmargin=0pt,noitemsep,topsep=0pt]
\item \textbf{Framework}. We formalize typed LM workflows as probabilistic programs: each learned hyperedge is an unnormalized conditional distribution that assigns zero mass to outputs violating type contracts.
\item \textbf{Theory}.
We propose a tractable and theoretically-grounded training algorithm, \model STaR. We prove that our optimization objective, while computationally efficient, correctly converges to the ideal solution as the model learns to become type-compliant. Specifically, we show that the bias in our gradient approximation vanishes as the model's adherence to type constraints increases during training (\cref{thm:unnormalizedisgoodenough,thm:boundingmlegradients}).
\item \textbf{Practice}. Across QA, structured generation, and classification tasks that require heavy reasoning (MGSM, MGSM-SymPy, FinQA, MuSR) and model families (Gemma, Qwen), \model s consistently outperform strong DSPy prompt-optimization baselines. Gains are largest when (1) base models are smaller and (2) tasks require strict structure. For example, on MGSM-SymPy with a Gemma 27B model, \model s achieve $\mathbf{75.9}$ vs. $\mathbf{57.1}$; on FinQA, $\mathbf{34.0}$ vs.\ $\mathbf{12.7}$ (Gemma 27B) and $\mathbf{24.7}$ vs.\ $\mathbf{12.0}$ (Qwen 3 8B). With a Gemma 7B model, MGSM improves from $\mathbf{1.6}$ to $\mathbf{27.3}$, FinQA from $\mathbf{0.7}$ to $\mathbf{9.7}$, and MuSR from $\mathbf{36.5}$ to $\mathbf{62.6}$.
\end{itemize}

\paragraph{Summary of results.} (1) Gradient-based adaptation \emph{within} typed workflows is markedly more effective than discrete prompt search for structured tasks. (2) Flexible training- and test-time posterior inference help performance. (3) Empirically, estimated type compliance mass $\mathcal{Z}_{\vtheta}$ rises rapidly during training and correlates with accuracy, supporting our theoretical justification for the unnormalized objective.

\section{Type-Compliant Adaptor Cascades}
\label{sec:method}

The core idea of \model s is to decompose a task into a hypergraph of interconnected transformations. Formally, a \model\, is represented as a directed acyclic hypergraph (DAH) $C = (\mathbf{Z}, \mathbf{E})$.\footnote{We use a reasoning workflow that generates domain-specific code, illustrated in \cref{fig:cascade}, as a running example. The task is to take a math question in English (input type \texttt{Q\_en}), generate a step-by-step rationale (intermediate type \texttt{R}), convert the rationale into a formal arithmetic expression (intermediate type \texttt{E}), and finally, have a deterministic function evaluate this expression to produce the answer (output type \texttt{A}). This section formalizes how such an intuitive sketch is realized within the \model\, framework.} The acyclic constraint ensures that the workflow has a well-defined topological order for execution and guarantees termination of the generative process.

\paragraph{Nodes.} The nodes $\mathbf{Z} = \{\mathbf{z}_1, \mathbf{z}_2, \ldots, \mathbf{z}_M\}$ in a \model\, act as containers for typed data. Each node $\mathbf{z}_m$ is associated with a specific data type $\tau \in \mathcal{T}$, and holds string representations $\in \Sigma^*$ for $\tau$-typed objects. Special nodes are designated as the \textbf{input node} $\mathbf{z}_1$ and the \textbf{output node} $\mathbf{z}_2$ (\emph{e.g.}, holding the initial question of type \texttt{Q\_en} and the final answer of type \texttt{A} in \cref{fig:cascade}, respectively).

\paragraph{Hyperedges.} Hyperedges $\mathbf{E} = \{e_1, e_2, \ldots, e_K\}$ define the transformations between nodes. A hyperedge $e_k$ connects a set of source nodes $S_k \subseteq \mathbf{Z}$ (its inputs) to a set of target nodes $T_k \subseteq \mathbf{Z}$ (its outputs). Transformations in \model s can be either learnable (LM adaptors) or fixed (deterministic algorithms):

\begin{itemize}[leftmargin=0pt,noitemsep,topsep=0pt]
    \item {\bf LM adaptor hyperedges.} These are stochastic transformations implemented by PEFT-adapted LMs. An adaptor $(\tau_i, \tau_o, \vtheta)$ defines an unnormalized distribution over $\vy \in \Sigma^*$ given input string $\vx$:\footnote{This distribution may be unnormalized because while $p_{LM}$  is a distribution over all strings, \cref{eq:adaptor_prob} restricts the support to only strings that are valid instances of $\tau_o$. Thus, the total probability mass may sum to less than $1$ if the LM assigns probability to invalid strings.}
\begin{align}
    \label{eq:adaptor_prob}
    \tilde{p}(\vy \mid \vx; \vtheta) &=  p_{LM}(\vy \mid \vx ; \vtheta) \mathbb{I}(\mathbf{z}_{t} \in \mathrm{valid}(\tau_o)),
\end{align}
where $p_{LM}(\cdot \mid \vx ; \vtheta)$ is a normalized  distribution over strings, conditioned on $\tau_i$-typed string representation $\vx$, and parametrized by adaptor parameters $\vtheta$, and $\mathrm{valid}(\tau_o) \subseteq \Sigma^*$ is the set of strings that represent valid $\tau_o$-typed objects (we will further discuss them in \cref{sec:interface_ops}). 

\item {\bf Deterministic algorithm hyperedges.} These are fixed, non-learnable transformations, such as a self-contained Python function. A deterministic algorithm $f$ maps an input object of type $\tau_i$ to an output object of type $\tau_o$.
Under the probabilistic view, we represent them as $\delta$ distributions:
\begin{align}
    \label{eq:deterministic_prob}
    \tilde{p}(\vy \mid \vx; f) &=  \delta_{\texttt{canon}(f(\texttt{parse}(x, \tau_i)))}(y)
\end{align}
where $\texttt{canon}$ (see \cref{sec:interface_ops}) produces a canonicalized string for an object, and $\texttt{parse}$ converts strings back to typed objects.
\end{itemize}

\subsection{Interfacing LLMs with Typed Data: Parsing and Canonicalization}
\label{sec:interface_ops}

A crucial subtlety in integrating LLMs into typed workflows is bridging their native string-based operation with typed data, which is typically handled by data validation libraries such as Pydantic\footnote{\url{https://github.com/pydantic/pydantic}} and LangFun.\footnote{\url{https://github.com/google/langfun}. Examples of generated prompts are listed in \cref{sec:prompts-generated-by-langfun}.} Here we formalize the conversion under the \model\,formalism as two operations \texttt{parse} and \texttt{canon}:

\paragraph{Parsing ($\texttt{parse}$).}
When an LM adaptor produces an output string $\vy$ intended to represent an object of type $\tau_o$, this string is validated and converted into a usable typed object by the algorithm $\texttt{parse}: \Sigma^* \times \mathcal{T} \rightarrow \mathcal{O} \cup \{\text{error}\}$.\footnote{We note that while primitive data types (\emph{e.g.}, Python types \texttt{str} and \texttt{list}) appear in common workflows, \texttt{parse} can be any computable function, and can be leveraged by a practitioner to implement complex business logic. For example, one can define a Python custom type \texttt{CoherentDialog} where valid objects are strings deemed coherent by an external LLM-backed classifier, and adapt LM adaptors in a \model\,to generate and work with such objects. Implementation details are further discussed in \cref{sec:implementation}.\label{ft:validation-logic}}
For example, in \cref{fig:cascade},  $\mathbf{z}_5$ has the deterministic function $e_4$ as an outgoing edge. During execution of the probabilistic program, $\texttt{parse}(\mathbf{z}_5, \texttt{E})$ attempts to convert $\mathbf{z}_5$ into a SymPy expression object (typed \texttt{E}).  If the conversion fails, an error is signaled. 
For convenience, we use $\mathrm{valid}(\tau) = \{ \texttt{parse}(\vy, \tau) \neq \text{error} \mid \vy \in \Sigma^* \}$ to denote valid string representations of $\tau$.

\paragraph{Canonicalization ($\texttt{canon}$).}
Conversely, inputs of LM adaptor hyperedges must be converted into a consistent string format that the adaptor expects. The $\texttt{canon}: \mathcal{O} \rightarrow \Sigma^*$ operation maps a typed object to a unique string representation \Dash we call such strings \emph{canonicalized}. The invertibility of $\texttt{canon}$ (\emph{i.e.}, $\texttt{parse}(\texttt{canon}(o), \tau_o) = o$) %
in turn ensures that deterministic hyperedges have support over only one string given a valid input, eliminating spurious ambiguity \citep{Cohen2012EliminationOS}.

\subsection{\model s As Programs And Distributions}
\label{sec:tacs_operation_prob}

\model s admit both a program view, and also a probabilistic view\footnote{These two views are also summarized in \cref{tab:dual-semantics}.}:
\begin{itemize}[leftmargin=0pt,noitemsep,topsep=0pt]
\item {\bf \model s are probabilistic programs.} Executing a \model\, in the forward direction involves processing data through the hypergraph, respecting the topological order of nodes and hyperedges. Using our running example from \cref{fig:cascade}: the process traverses the hypergraph, starting at the input variable $\mathbf{z}_1$ (typed \texttt{Q\_en}), and ending at the output variable $\mathbf{z}_2$ (typed \texttt{A}). A general process is described in \cref{alg:cascade-forward}.
\item {\bf \model s are also probability distributions.} \model s also define unnormalized joint probability distributions over all node assignments $\mathbf{Z}^* = (\mathbf{z}^*_1, \mathbf{z}^*_2, \ldots, \mathbf{z}^*_M)$. This score reflects the plausibility of a complete execution trace according to the model's components:
\begin{align}
    \log \tilde{p}_{\vtheta}(\mathbf{Z}^*) &= \sum_{k} \log \tilde{p}_{\vtheta}( \{\mathbf{z}^*_t\}_{t \in T_k} \mid \{ \mathbf{z}^*_s \}_{s \in S_k} ; e_k ),
    \label{eq:joint_prob}
\end{align}
where $\vtheta$ represent all adaptor parameters used in the \model , and $\tilde{p}_{\vtheta}(\cdot|\cdot; e_k)$ is the conditional probability defined by the LM adaptor (\cref{eq:adaptor_prob}) or deterministic algorithm (\cref{eq:deterministic_prob}) associated with $e_k$. The unnormalized distribution view connects \model s to the broader family of language model cascades \citep{dohan2022languagemodelcascades}, but with the key distinction that \model s are designed for end-to-end adaptation.
\end{itemize}

\paragraph{Estimating unnormalized marginal probabilities.} LM adaptors in a \model\,can be used as proposal distributions to get an importance sampling estimate of the unnormalized marginal probability. Let $\mathbf{z}_m$ be a node coming out of an LM adaptor, an $N$-sample estimate of the unnormalized probability that $\mathbf{z}_m$ equals $c$: $\tilde{p}(\mathbf{z}_m = c ; \vtheta)$ is:
\begin{align}
    \hat{\tilde{p}}_{\mid \mathbf{z}_1}(m, c, N) &= \sum_{n=1}^{N} \left[ \frac{{p}_{LM}(\mathbf{z}_m = c; \vtheta)}{N \cdot {p}_{LM}(\mathbf{z}_m = \mathbf{z}_m^{(n)}; \vtheta)} \right]
    \label{eq:unnormalized-marginal-estimate}
\end{align}
where $\mathbf{z}_m^{(n)}$ is the $n$-th sample of $\mathbf{z}_m$ (possibly drawn using \cref{alg:cascade-forward}). \Cref{eq:unnormalized-marginal-estimate} is an unbiased importance sampling estimate of the unnormalized probability ${\tilde{p}}(\mathbf{z}_m = c \mid \mathbf{z}_1 ; \vtheta)$ (since $\mathrm{supp}(\tilde{p}) \subseteq \mathrm{supp}(p_{LM})$). In general, $\mathbf{z}_m$ has an infinite support, making the \emph{normalized} probability ${p}(\mathbf{z}_m = c \mid \mathbf{z}_1 ; \vtheta)$ intractable. In the special case that $\mathbf{z}_m$  has finite support, \cref{eq:unnormalized-marginal-estimate} can be used to estimate the \emph{normalized} marginal probability $\hat{p}(\mathbf{z}_m = c \mid \mathbf{z}_1 ; \vtheta) = \frac{ \hat{\tilde{p}}_{\mid \mathbf{z}_1}(m, c, N) }{\sum_{c'} \hat{\tilde{p}}_{\mid \mathbf{z}_1}(m, c', N) }$. We leverage \cref{eq:unnormalized-marginal-estimate} to estimate normalized output probabilities $p(\mathbf{z}_2 \mid \mathbf{z}_1 ; \vtheta)$, for ranking classification outputs in \cref{sec:amortized-exp}.

\section{Adapting \model s}
\label{sec:training}

Since \model s generally define distributions over unobserved (latent) intermediate variables, Monte Carlo Expectation-Maximization (MC-EM) algorithms \citep{Wei1990AMC} provide a suitable training paradigm for marginalized likelihood maximization.\footnote{We acknowledge that another reasonable approach for training \model s is reinforcement learning, and note the connection between \model STaR and RL in \cref{sec:problem}.} MC-EM algorithms iteratively refine model parameters by alternating between an E-step (sampling latent variables) and an M-step (optimizing parameters based on these samples). The Self-Taught Reasoner (STaR) algorithm \citep{zelikman2022} is a notable instance of MC-EM. We generalize STaR to the \model\ framework for workflows with arbitrarily typed inputs and outputs, resulting in the \model STaR algorithm.

\subsection{\model STaR}

\label{sec:tac-star}

The \model STaR algorithm (\cref{alg:model-star}) employs an iterative MC-EM approach to train the parameters $\vtheta$ of the type-compliant LM adaptors within a \model\ $C$. As with the original STaR algorithm, \model STaR alternates between E- and M-steps:
\begin{itemize}[leftmargin=0pt,noitemsep,topsep=0pt]
    \item {\bf E-step: Sampling Latent Variables.} We first try to execute the \model\ $C$ as a probabilistic program under the \texttt{forward} algorithm (\cref{alg:cascade-forward}). If \texttt{forward} succeeds, we have a complete assignment of values $\mathbf{Z}^* = (\mathbf{z}_1^*, \mathbf{z}_2^*, \ldots, \mathbf{z}_M^*)$ for all nodes in the \model\,$C$. and can proceed to M-step. Otherwise, we attempt a \textbf{rationalization heuristic} step. Inspired by the original STaR algorithm which conditions on the correct answer in the second attempt, we construct a `fallback' \model , whose input node takes $(x^*, y^*)$ as input, with the rest of the workflow unchanged. This essentially asks `\emph{what intermediate steps would lead from $x^*$ to $y^*$?}', analogous to the inverse rendering problem \citep{ritchie23}. A forward pass is then executed on this new \model\,to sample $(\mathbf{z}_2, \ldots , \mathbf{z}_M)$, now conditioned on both the original input $x^*$ and the desired output $y^*$. This encourages the generation of latent intermediate steps that are consistent with the correct final answer.
    \item {\bf M-step: Parameter Optimization.} EM-style algorithms generally do MLE updates on samples collected in the E-step. As \model s are generally unnormalized models, proper MLE updates require computing partition function gradients. Denoting the partition function summing all possible assignments as $\mathcal{Z}_{\vtheta} = \sum_{\mathbf{Z}'} \tilde{p}_{\vtheta}(\mathbf{Z}')$, the gradient of the $\log$-likelihood $\mathcal{L} = \log p(\mathbf{Z}^*)$ is:
\begin{align}
    \nabla_{\vtheta} \mathcal{L} = \nabla_{\vtheta} \log \tilde{p}_{\vtheta}(\mathbf{Z}^*) - \nabla_{\vtheta} \log \mathcal{Z}_{\vtheta}.
    \label{eq:grad-true-loss}
\end{align}
Estimation of the $\log$ partition function's gradients $\nabla_{\vtheta} \log \mathcal{Z}_{\vtheta}$ is typically expensive and can have high variance \citep{Goodfellow-et-al-2016-partition}. We thus drop this term, and optimize for the unnormalized $\log$-likelihood $\mathcal{L}'(\vtheta) = \log \tilde{p}_{\vtheta}(\mathbf{Z}^*)$ instead.\footnote{{\bf Remark on efficiency.} Since gradients of the $\log$ unnormalized probability decompose linearly as $\nabla_{\vtheta} \left( \log \tilde{p}_{\vtheta}(\mathbf{Z}^*) \right) = \sum_{k} \nabla_{\vtheta} \log \tilde{p}_{\vtheta}( \{\mathbf{z}^*_t\}_{t \in T_k} \mid \{ \mathbf{z}^*_s \}_{s \in S_k} ; e_k )$, computation of adaptors' gradients can be parallelized easily.
This embarrassingly parallel structure ensures computational scalability, allowing the M-step to be efficiently distributed across available compute resources. 
\Cref{alg:cascade-backward} computes $\log \tilde{p}_{\vtheta}(\mathbf{Z}^*)$ and its gradients $\nabla_{\vtheta} \log \tilde{p}_{\vtheta}(\mathbf{Z}^*)$. These gradients are then used in a standard gradient-based optimization algorithm to update $\vtheta$.}
\end{itemize}

\paragraph{Tractable optimization via compliance.} While ignoring the partition function gradient generally leads to biased gradient estimation, the \model\,formalism ensures this strategy is both tractable and robust. This becomes evident as we rewrite $\mathcal{L}'(\vtheta) = \mathcal{L}(\vtheta) + \log \mathcal{Z}_{\vtheta}$:
optimizing the unnormalized likelihood $L'(\vtheta)$ is equivalent to jointly maximizing the normalized likelihood $\mathcal{L}(\vtheta)$ and the model's type compliance (the partition function $\log \mathcal{Z}_{\vtheta}$ is maximized at $\log \mathcal{Z}_{\vtheta} = 0$ when  $\vtheta$ is well-specified).
This approach is justified theoretically under the assumption that the adapted models can perfectly model type-valid outputs (\emph{i.e.}, the model family is well-specified):\footnote{We refer the reader to \cref{sec:partition-function-discussion} for proofs of formal statements in this section.}
\begin{restatable}{theorem}{unnormalizedisgoodenough}
Let $\vTheta$ be the entire parameter space and let $\vTheta' \subseteq \vTheta$ be the subset of well-specified parameters. Assume $\vtheta^*$ uniquely maximizes the normalized likelihood $p_\vtheta(\mathbf{z}_{2..M}|\mathbf{z}_1)$ and resides $\in \vTheta'$.
Then, $\hat{\vtheta} = \argmax_{\vtheta \in \vTheta} \tilde{p}_\vtheta(\mathbf{z}_{2..M}|\mathbf{z}_1) \implies \hat{\vtheta} = \vtheta^*$.
\label{thm:unnormalizedisgoodenough}
\end{restatable}

Moreover, while optimizing $\mathcal{L}'(\vtheta)$ introduces a bias by ignoring the gradient term $\nabla_{\vtheta} \log \mathcal{Z}_{\vtheta}$, this bias is bounded below a constant multiplicative factor of $(1 - \mathcal{Z}_{\vtheta})$ under the common assumption that $\Vert \nabla_{\vtheta} p_{LM}(\cdot \mid \vx ; \vtheta ) \Vert$ is uniformly bounded:
\begin{restatable}{theorem}{boundingmlegradients}
Let $\vtheta = \{ \vtheta_{1} \ldots \vtheta_{K} \}$ be the union of a $K$-adaptor \model 's LM adaptor parameters . If $\forall \mathbf{z}_{k,1} \in \Sigma^*, \mathbf{z}_{k,2} \in \Sigma^*,  \Vert \nabla{\vtheta} \left( \sum \log p_{LM}(\mathbf{z}_{k,2} \mid \mathbf{z}_{k,1} ; \vtheta ) \right) \Vert_{\infty} \leq G$, then $\nabla_{\vtheta} \log \mathcal{Z}_{\vtheta} \leq 2 G (1 - \mathcal{Z}_{\vtheta}) $.
\label{thm:boundingmlegradients}
\end{restatable}
\cref{thm:unnormalizedisgoodenough,thm:boundingmlegradients} provide theoretical assurance that if the model achieves high type compliance  as we optimize for $\mathcal{L}'(\vtheta) = \mathcal{L}(\vtheta) + \log \mathcal{Z}_{\vtheta}$, the \model STaR M-step update approaches true MLE update. Empirically, we observe \model STaR rapidly drives $\mathcal{Z}_{\vtheta}$ towards $1$ (\cref{sec:structural-compliance}).

\subsection{Amortized \model STaR}
\label{sec:amortized_model_star}

Amortized \model STaR (\cref{alg:amortized-model-star}) generalize the `fallback' rationalization heuristic in \model STaR as parametric inference networks \citep{Kingma2014,pmlr-v32-mnih14}, jointly trained to approximate the true posterior given observed input and outputs. By learning to propose better, task-adapted latent variable configurations, Amortized \model STaR can hopefully lead to more efficient training and potentially better performance of the model \model .
For model \model\,$C$ with nodes $\mathbf{z}_1 \ldots \mathbf{z}_M$, we construct an inference network \model\,$C'$ with nodes $\mathbf{z}'_1 \ldots \mathbf{z}'_M$, which is trained alongside with $C$. %
In this work, we construct $\mathbf{z}'_2 \ldots \mathbf{z}'_M$ to have the same types as $\mathbf{z}_2 \ldots \mathbf{z}_M$, except for its input node $\mathbf{z}_1'$, which has a type to represent the input-output pair $(x^*, y^*)$. Moreover, we construct $C'$ so that every adaptor hyperedge $e_k$ in $C$ has a counterpart $e_k'$ in $C'$ that is additionally conditioned on $\mathbf{z}_1'$. We train $C'$ alternately with $C$, with the goal of making the unnormalized distribution of $C'$ over its nodes except for $\mathbf{z}_1'$ approximate the posterior over $C$'s intermediate nodes, conditioning on $(x^*, y^*)$ observations. Denoting the unnormalized distribution of $C'$ as $\tilde{q}_{\vphi}$ parametrized by adaptors' parameters $\vphi$, we hope to learn $\vphi$ such that $\tilde{q}_{\vphi}( \mathbf{z}'_m \mid \mathbf{z}'_1 = \texttt{canon}((x^*, y^*))) \approx {p}_{\vtheta}( \mathbf{z}_m \mid \mathbf{z}_1 = x^*_c, \mathbf{z}_2 = y^*_c)$, where $x^*_c = \texttt{canon}(x^*)$, $y^*_c = \texttt{canon}(y^*)$, $\forall m \in [2 .. M]$. Approximating the posterior ${p}_{\vtheta}( \mathbf{z}_m \mid \mathbf{z}_1 = \texttt{canon}(x^*), \mathbf{z}_2 = \texttt{canon}(y^*))$ as $\hat{p}$ using self-normalized multiple importance sampling \citep{veach1995}, we optimize $\vphi$ to minimize $\mathrm{KL}[\hat{p} || \tilde{q}_{\vphi}]$ following \cite{Bornschein2014ReweightedW,lin-eisner-2018-neural}.

\section{Experiments}

\label{sec:experiments}

To empirically validate \model\,models, we conduct QA, code-like structured generation, and classification experiments on subsets of MGSM \citep{shi2022}, FinQA \citep{chen-etal-2021-finqa}, and MuSR \citep{sprague2024musr} datasets,\footnote{We defer the study of how different \model\,patterns affect performance to \cref{sec:additional-studies-pattern-design}, where we expand our experiments to include HotPotQA tasks \citep{yang-etal-2018-hotpotqa}.} adapting both instruction-tuned Gemma 7B and Gemma 2 27B (referred to as \texttt{gemma-1.1-7b-it} and \texttt{gemma-2-27b-it}) \citep{gemmateam2024gemma2improvingopen}, and Qwen 3 8B models (\texttt{Qwen3-8B}) \citep{qwen3}.
We aim to answer the following research questions: 
\begin{itemize}[leftmargin=0pt,noitemsep,topsep=0pt]
    \item {\bf (\cref{sec:effectiveness-against-prompt-optimization}) Are \model s competitive against existing approaches?} \model s differ from existing LM adaptation approaches in two major ways: 1) \model s support gradient-based learning in a unified probabilistic programming framework (when compared against prior prompt optimization-focused LM programming frameworks such as DSPy); and 2) \model s support structured workflows by design (when compared to the original STaR algorithm). We hypothesize that such difference translates into meaningful performance improvements.
    \item {\bf (\cref{sec:amortized-exp}) Is exploiting \model s' probabilistic flexibility effective?} 
    Probability models (such as \model s) benefit from the decoupling of probabilistic modeling and inference procedures, allowing conditioning on additional observations \emph{a posteriori}. We evaluate whether exploiting this flexibility is effective in two scenarios: 1) We compare Amortized \model STaR (\cref{sec:amortized_model_star}), which conditions on the output variable to learn a better proposal distribution for training, against the standard (unconditioned) \model STaR; and 2) We evaluate \model s on a classification task, comparing the performance of unconstrained generation against a renormalized classifier that evaluates and normalizes the conditional probability of each possible output.
    \item {\bf (\cref{sec:structural-compliance}) Does the model achieve high type compliance?} A key theoretical result (\cref{sec:tac-star}) is that the soundness and near-optimality of the \model STaR optimization strategy rely on the model learning to comply with the workflow's type constraints (\emph{i.e.}, driving the partition function $\mathcal{Z}_{\vtheta} \rightarrow 1$). As type compliance increases, the gap between the tractable unnormalized likelihood and the true normalized likelihood ($\log \mathcal{Z}_{\vtheta}$) closes. We estimate how $\mathcal{Z}_{\vtheta}$ over \model STaR epochs to verify that this gap is negligible after training.
\end{itemize}
\subsection{Experiment Setup}

\label{sec:setup}

We provide an overview of our \model\,and baseline DSPy setups below:

\begin{itemize}[leftmargin=0pt,noitemsep,topsep=0pt]
\item {\bf \model s.} We parametrize \model\,adaptors to take the form of rank-$1$ LoRA models \citep{hu2022lora} on the attention weights, with $573,440$; $1,413,120$; and $958,464$ parameters per adaptor for \texttt{gemma-1.1-7b-it}, \texttt{gemma-2-27b-it} and \texttt{Qwen3-8B} respectively. For $\texttt{parse}$ and $\texttt{canon}$ implementations (\cref{sec:interface_ops}), we leverage the LangFun library, which prompts LLMs to generate Python classes and objects, and parses their responses. LoRA weights are initialized (`zero-init') following \citet{hu2022lora}.
\item {\bf DSPy.} We conduct prompt-optimizing baseline experiments under DSPy, with base models served on vLLM. We subclass \texttt{dspy.Signature} to represent training examples, with property names and types identical to their \model\, counterparts (some examples are listed in \cref{sec:dspy-setup}). We employ XGrammar \citep{dong2024xgrammar} for schema-based constrained decoding for all experiments. We implement two types of reasoning workflows for all tasks: 1) the native \texttt{dspy.ChainOfThought} module, and 2) an explicitly two-step composite module that resembles {\bf cot-cascade-structure} patterns under \model s. We experiment with various prompt optimization configurations under \texttt{dspy.MIPROv2} \citep{opsahl-ong-etal-2024-optimizing} and \texttt{dspy.BootstrapFewShotWithRandomSearch} \citep{khattab2024dspy}.
\end{itemize}

We conduct experiments of $5$ reasoning-heavy tasks, on subsets from datasets MGSM\footnote{The MGSM-SymPy task uses the same problems of MGSM, but additionally restrict the outputs to be rational expressions under SymPy. This variant was specifically included to test the framework's ability to generate and comply with highly structured, code-like output.} \citep{shi2022}, FinQA \citep{yang-etal-2018-hotpotqa}, HotPotQA \citep{yang-etal-2018-hotpotqa} and MuSR \citep{sprague2024musr}  respectively. 
Details of experiment setup are described in \cref{sec:setup-appendix}.
\subsection{Comparison against prompt-optimizing and untyped STaR baselines.}
\label{sec:effectiveness-against-prompt-optimization}

\Cref{tbl:prompt-optimization-comparison} lists MGSM, MGSM-SymPy, FinQA, and MuSR results from best-performing \model s and DSPy models.
In addition,  we compare the untyped (original) STaR against typed \model\,results on MGSM on Gemma models.

\paragraph{\model s are competitive against prompt-optimizing baseline methods.} We observe that \model s consistently and significantly outperform DSPy baselines in every setting. The performance gap is especially wide when 1) the base model is smaller, and 2) the task involves structured inputs (FinQA) or structured outputs (MGSM-SymPy).\footnote{We also compare between \model STaR-adapted and un-adapted models on the same LangFun prompts in \cref{sec:effectiveness-of-latent-structures}, and find that \model STaR consistently outperforms the un-adapted counterparts.}

\begin{figure}[ht]
    \centering
    \subfloat[FinQA]{
    \resizebox{.35\linewidth}{!}{
    \begin{tabular}{crr}
    \toprule
        Base Model & DSPy & \model  \\
        \midrule
        \texttt{gemma-1.1-7b-it}  & $0.7\%$ & $\mathbf{9.7\%}$ \\
        \texttt{gemma-2-27b-it}  & $12.7\%$ & $\mathbf{34.0\%}$ \\
        \texttt{Qwen3-8B}  & $12.0\%$ & $\mathbf{24.7\%}$ \\
    \bottomrule
    \end{tabular}
    }
    } \quad
    \subfloat[MuSR]{
    \resizebox{.35\linewidth}{!}{
    \begin{tabular}{crr}
    \toprule
        Base Model & DSPy & \model  \\
        \midrule
        \texttt{gemma-1.1-7b-it}  & $36.5\%$ & $\mathbf{62.6\%}$ \\
        \texttt{gemma-2-27b-it}  & $51.5\%$ & $\mathbf{65.0\%}$ \\
        \texttt{Qwen3-8B}  & $61.5\%$ & $\mathbf{63.7\%}$ \\
    \bottomrule
    \end{tabular}
    }
    } \quad
    \subfloat[MGSM]{
    \resizebox{.35\linewidth}{!}{
    \begin{tabular}{crrr}
    \toprule
        Base Model & DSPy & \model & STaR  \\
        \midrule
        \texttt{gemma-1.1-7b-it}  & $1.6\%$ & $\mathbf{27.3\%}$ & $10.5\%$ \\
        \texttt{gemma-2-27b-it}  & $81.9\%$ & $\mathbf{82.2\%}$ & $76.9\%$ \\
    \bottomrule
    \end{tabular}
    }
    \label{tbl:comparison-mgsm}
    } \quad
    \subfloat[MGSM-SymPy]{
    \resizebox{.35\linewidth}{!}{
    \begin{tabular}{crr}
    \toprule
        Base Model & DSPy & \model  \\
        \midrule
        \texttt{gemma-2-27b-it}  & $57.1\%$ & $\mathbf{75.9\%}$ \\
    \bottomrule
    \end{tabular}
    }
    }
    \caption{Comparison between best performing prompt-optimizing methods under DSPy and \model s (full results can be found in \cref{sec:per-language-mgsm-results,sec:tac-per-task-musr-results,sec:dspy-per-task-musr-results,sec:dspy-per-language-mgsm-results,sec:dspy-finqa-results}). We report the best DSPy result for each task.}
\label{tbl:prompt-optimization-comparison}
\end{figure}

\begin{figure*}[h]
    \centering
    \subfloat[Average estimate $\log \mathcal{Z}_{\vtheta}$ over validation set inputs versus \# of \model STaR epochs over MGSM languages. Note that later epochs (as early as epoch $5$) do not have samples from all languages, as some languages early-stopped.]{\includegraphics[width=.35\linewidth,valign=b]{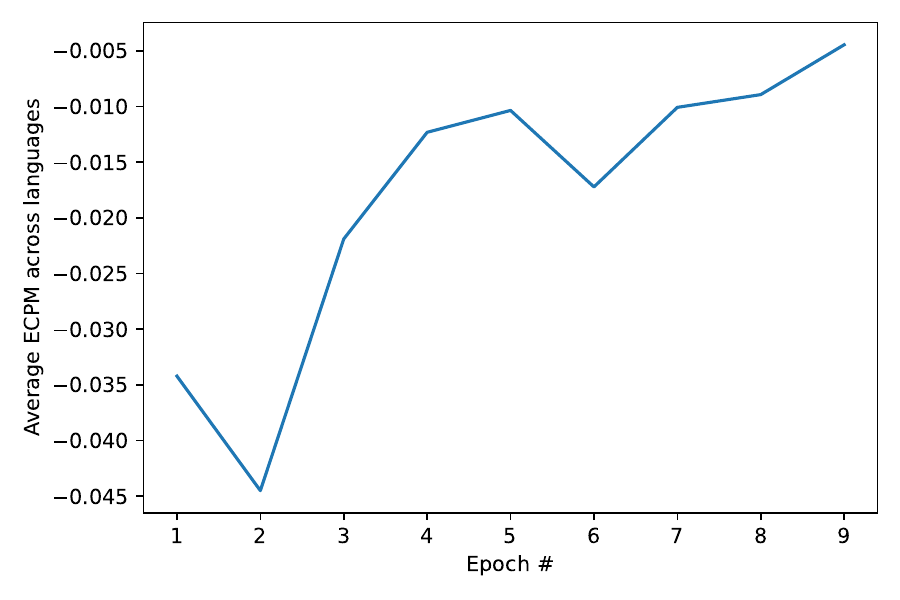}\label{fig:log-z-estimation}} \qquad
    \subfloat[Average MGSM training data parsing failure rate vs \# of epochs of \model STaR on \texttt{gemma-1.1-7b-it}. The pattern is {\bf cot-cascade-structure}.]{
    \begin{tabular}[b]{cr}
    \toprule
        At the end of epoch & Failure rate \\
        \midrule
        1 & $83.0\%$ \\
        2 & $1.0\%$ \\
        3 & $1.6\%$ \\
        4 & $0.4\%$ \\
        \bottomrule
    \end{tabular}
    \label{tbl:error-rate}
    }
    \caption{Type compliance during \model\,training.}
\end{figure*}

\paragraph{\model STaR compares favorably against the original STaR algorithm on unstructured data.} On the MGSM task (\cref{tbl:comparison-mgsm}), the original (untyped) STaR algorithm scored an average accuracy of $76.9$ and $10.5$ (from \texttt{gemma-2-27b-it} and \texttt{gemma-1.1-7b-it} respectively), lower than variants of reasoning \model\, patterns on the same dataset. This demonstrates that the structured, typed approach of \model s improves performance over the untyped STaR baseline.

\subsection{Flexible Posterior Inference Helps \model\,Performance.}

\label{sec:amortized-exp}

\begin{figure}[h]
\centering
\subfloat[Comparison between \model STaR and Amortized \model STaR on {\bf cot-cascade-structure} / \texttt{gemma-2-27b-it}.]{
\begin{tabular}[b]{lrr}
\toprule
Task   & \model STaR & Amortized \model STaR \\
\midrule
MGSM     & {82.2} & {82.4}  \\
FinQA & {36.0} & {41.7} \\
HotPotQA & {32.0} & {34.0} \\
\bottomrule
\end{tabular}
\label{tab:amortized-star-comparison}
} \quad
\subfloat[Comparison between classification and unconstrained generation results on MuSR.]{
\begin{tabular}[b]{p{.22\linewidth}rr}
\toprule
Base Model   & Cla. &  Gen. \\
\midrule
{\texttt{gemma-1.1-7b-it}}     & {62.6} & {62.1}  \\
{\texttt{gemma-2-27b-it}} & {65.0} & {51.6} \\
\bottomrule
\end{tabular}
\label{tab:musr-comparison}
}
\caption{Comparison between `default' and more informative inference methods.}
\label{tab:flexible-inference}
\end{figure}

\paragraph{Amortized inference at training time is effective.} The Amortized \model STaR algorithm (\cref{sec:amortized_model_star}) brings consistent improvement over vanilla \model STaR on $3$ tasks (\cref{tab:amortized-star-comparison}). Notably, the gains are most substantial on FinQA ($+5.7$ points). This suggests that amortized inference is particularly valuable for complex tasks where the initial sampling or fixed rationalization heuristics struggle to find valid latent traces, allowing the model to learn a more effective inference strategy.

\paragraph{Classification with renormalized posterior at inference time is effective.} We renormalize importance sampling estimates (\cref{eq:unnormalized-marginal-estimate}) to estimate the output label posterior $p_{\vtheta}(\mathbf{z}_2 \mid \mathbf{z}_1 )$ for the MuSR classification task, and output the label with highest probability. \Cref{tab:musr-comparison} shows that the renormalized-posterior classifier outperforms unconstrained generation on both \texttt{gemma-1.1-7b-it} and \texttt{gemma-2-27b-it} base models.

\subsection{\model\,models rapidly achieve high type compliance.}

\label{sec:structural-compliance}

We argued in \cref{sec:tac-star} that optimizing the unnormalized likelihood drives the model towards structural compliance.
The average MGSM parsing error rate during training (\cref{tbl:error-rate}) suggests that \model s learn compliance fast. %
We further empirically verify this by estimating the partition function $\mathcal{Z}_{\vtheta}$ \Dash which represents the total probability mass the model assigns to type-compliant outputs (the Estimated Compliant Probability Mass, ECPM) \Dash throughout training.
We estimate $\log \mathcal{Z}_{\vtheta}$ on the validation sets of the MGSM benchmark during training of the {\bf cot-cascade-structure} pattern on \texttt{gemma-1.1-7b-it}. We sample $100$ generations of entire traces without type-compliant masking per input with temperature $=1$, top-p $=1$, and top-k set to the vocabulary size. \Cref{fig:log-z-estimation} shows that the model rapidly learns to comply with the type constraints. The average $\log \mathcal{Z}_{\vtheta}$ approaches $-0.005$ by epoch 9, corresponding to an ECPM of $\exp(-0.005) \approx 99.5\%$, and thus confirms that the degree of misspecification $(1 - Z_{\vtheta})$ is negligible. Since the difference between unnormalized and normalized likelihood gradients is bounded by a multiplicative factor of $(1 - Z_{\vtheta})$ (\cref{thm:boundingmlegradients}), our empirical estimates imply that the difference is indeed small at the end of training, and \model STaR M-step (\cref{sec:tac-star}) approaches the true MLE update. Moreover, since $\log \mathcal{Z}_{\vtheta}$ is the difference between normalized and unnormalized likelihoods, the small magnitude suggests it is practical to do model selection with unnormalized likelihood directly, after a few epochs of training.

\section{Limitations and Future Work}

While the Type-Compliant Adaptation Cascades (\model s) framework demonstrates significant advantages in adapting LLMs for structured workflows, several limitations point toward important directions for future research.

\paragraph{Tractability and Theoretical Assumptions.} A primary limitation stems from the trade-off made for computational tractability in our training algorithms. Both \model STaR and Amortized \model STaR omit the gradient of the log-partition function ($\nabla_{\vtheta} \log \mathcal{Z}_{\vtheta}$) during the M-step (\cref{sec:tac-star}). While this approach is justified theoretically by \cref{thm:unnormalizedisgoodenough,thm:boundingmlegradients} and validated empirically in the work (\cref{sec:structural-compliance}), this justification may no longer hold for more difficult problems, or if we further reduce the adaptors' capacity. Future work could explore efficient methods for estimating or bounding the partition function gradient, or investigate advanced variational or contrastive training methods that relax the well-specified assumption.

\paragraph{Manual Workflow Design and Automated Synthesis.} Currently, the structure of a \model\,hypergraph \Dash the definition of intermediate types and the connections between transformations \Dash must be manually designed by the practitioner. While the declarative nature of \model s facilitates rapid iteration over different designs (as shown in \cref{sec:does-cascade-design-matter}), the framework does not yet support automatic structure discovery. This limitation presents a significant opportunity for automated workflow synthesis. Because a \model\,hypergraph is itself a formally defined, typed object, it is feasible to develop higher-level models trained to generate the necessary \model\,graph for a given task description. The foundational framework presented in this paper, which provides robust methods for adapting any given typed workflow via gradient-based optimization, is an essential prerequisite for enabling such a ``generate-then-execute'' paradigm.

\paragraph{Complexity and Latency of Type Constraints.} The \model\,framework relies on \texttt{parse} and \texttt{canon} functions to bridge the gap between LLM-generated strings and structured objects. For the common syntactic types used in our experiments (\emph{e.g.}, lists, expression trees), the overhead of parsing is negligible compared to LLM inference latency. However, the framework supports arbitrarily complex validation logic within the parse function, including semantic checks that might involve external calls (\emph{e.g.}, to a scoring model or validation API). While this enables stronger guarantees, complex semantic constraints could introduce significant latency, representing a trade-off that practitioners must navigate between execution speed and the strength of semantic guarantees.

\paragraph{Model and Task Generalizability.} Following previous work \cite{khattab2024dspy} in the LM programming literature, we validated \model s across several reasoning-heavy benchmarks. On two model sizes (7B and 27B), we obtained significant gains over strong prompt-optimization baselines. We also demonstrated applicability beyond QA through domain-specific code generation (MGSM-SymPy). To further substantiate the framework's generalizability, future work should expand evaluations to include a wider diversity of model architectures and broader task domains, such as creative tasks, interactive agents and multi-modal workflows, where the reliable, typed integration of LLMs and external tools is crucial. For example, \model s could provide a principled foundation for complex agentic systems, where we adapt models through unrolled tool-use feedback loops, with predictability guaranteed by the type-compliant nature of \model s. Another example is multimodal generation pipelines, where data of different modalities are processed and transformed in a trainable and programmatic workflow. We leave systematic exploration tu future work.

\section{Related Work}
\label{sec:related-work}

The challenge of adapting LLMs to complex problems involving structured workflows and type constraints intersects with several lines of research, including programmatic LM workflows, probabilistic programming, parameter-efficient fine-tuning, and constrained decoding. We defer a more extensive survey to \cref{sec:problem}.

\section{Conclusion}
\label{sec:conclusion}

We have presented Type-Compliant Adaptation Cascades (\model s), a novel probabilistic programming framework designed to empower ML practitioners to design trainable workflows that adapt to data. 
Our findings demonstrate that \model s' gradient-based learning paradigm is highly effective, consistently outperforming strong prompt-optimization baselines. Moreover, we also find flexible posterior inference of \model s at both training and inference time help with performance. We also find that empirically, the model learns to comply with type constraints fast in training, justifying the assumptions in our theoretical results.
These results underscore the versatility and efficacy of \model s as a scalable paradigm for adapting to complex, reasoning-heavy tasks.

\onecolumn 

\bibliographystyle{unsrtnat}
\bibliography{iclr2025_conference}

\begin{thebibliography}{58}
\providecommand{\natexlab}[1]{#1}
\providecommand{\url}[1]{\texttt{#1}}
\expandafter\ifx\csname urlstyle\endcsname\relax
  \providecommand{\doi}[1]{doi: #1}\else
  \providecommand{\doi}{doi: \begingroup \urlstyle{rm}\Url}\fi

\bibitem[Rosenfeld(2018)]{Rosenfeld2018}
Roni Rosenfeld.
\newblock {Two Decades of Statistical Language Modeling: Where Do We Go From Here?}
\newblock 6 2018.
\newblock \doi{10.1184/R1/6611138.v1}.
\newblock URL \url{https://kilthub.cmu.edu/articles/journal_contribution/Two_Decades_of_Statistical_Language_Modeling_Where_Do_We_Go_From_Here_/6611138}.

\bibitem[Brown et~al.(2020)Brown, Mann, Ryder, Subbiah, Kaplan, Dhariwal, Neelakantan, Shyam, Sastry, Askell, Agarwal, Herbert-Voss, Krueger, Henighan, Child, Ramesh, Ziegler, Wu, Winter, Hesse, Chen, Sigler, Litwin, Gray, Chess, Clark, Berner, McCandlish, Radford, Sutskever, and Amodei]{gpt3}
Tom Brown, Benjamin Mann, Nick Ryder, Melanie Subbiah, Jared~D Kaplan, Prafulla Dhariwal, Arvind Neelakantan, Pranav Shyam, Girish Sastry, Amanda Askell, Sandhini Agarwal, Ariel Herbert-Voss, Gretchen Krueger, Tom Henighan, Rewon Child, Aditya Ramesh, Daniel Ziegler, Jeffrey Wu, Clemens Winter, Chris Hesse, Mark Chen, Eric Sigler, Mateusz Litwin, Scott Gray, Benjamin Chess, Jack Clark, Christopher Berner, Sam McCandlish, Alec Radford, Ilya Sutskever, and Dario Amodei.
\newblock Language models are few-shot learners.
\newblock In H.~Larochelle, M.~Ranzato, R.~Hadsell, M.F. Balcan, and H.~Lin, editors, \emph{Advances in Neural Information Processing Systems}, volume~33, pages 1877--1901. Curran Associates, Inc., 2020.
\newblock URL \url{https://proceedings.neurips.cc/paper_files/paper/2020/file/1457c0d6bfcb4967418bfb8ac142f64a-Paper.pdf}.

\bibitem[Wei et~al.(2022{\natexlab{a}})Wei, Bosma, Zhao, Guu, Yu, Lester, Du, Dai, and Le]{wei2022finetuned}
Jason Wei, Maarten Bosma, Vincent Zhao, Kelvin Guu, Adams~Wei Yu, Brian Lester, Nan Du, Andrew~M. Dai, and Quoc~V Le.
\newblock Finetuned language models are zero-shot learners.
\newblock In \emph{International Conference on Learning Representations}, 2022{\natexlab{a}}.
\newblock URL \url{https://openreview.net/forum?id=gEZrGCozdqR}.

\bibitem[Khattab et~al.(2022)Khattab, Santhanam, Li, Hall, Liang, Potts, and Zaharia]{khattab2022demonstrate}
Omar Khattab, Keshav Santhanam, Xiang~Lisa Li, David Hall, Percy Liang, Christopher Potts, and Matei Zaharia.
\newblock Demonstrate-search-predict: Composing retrieval and language models for knowledge-intensive {NLP}.
\newblock \emph{arXiv preprint arXiv:2212.14024}, 2022.

\bibitem[Chase(2022)]{Chase_LangChain_2022}
Harrison Chase.
\newblock {LangChain}, October 2022.
\newblock URL \url{https://github.com/langchain-ai/langchain}.

\bibitem[Yao et~al.(2023)Yao, Zhao, Yu, Du, Shafran, Narasimhan, and Cao]{yao2023react}
Shunyu Yao, Jeffrey Zhao, Dian Yu, Nan Du, Izhak Shafran, Karthik Narasimhan, and Yuan Cao.
\newblock {ReAct}: Synergizing reasoning and acting in language models.
\newblock In \emph{International Conference on Learning Representations (ICLR)}, 2023.

\bibitem[Wu et~al.(2024)Wu, Bansal, Zhang, Wu, Li, Zhu, Jiang, Zhang, Zhang, Liu, Awadallah, White, Burger, and Wang]{wu2024autogen}
Qingyun Wu, Gagan Bansal, Jieyu Zhang, Yiran Wu, Beibin Li, Erkang Zhu, Li~Jiang, Xiaoyun Zhang, Shaokun Zhang, Jiale Liu, Ahmed~Hassan Awadallah, Ryen~W White, Doug Burger, and Chi Wang.
\newblock Autogen: Enabling next-gen {LLM} applications via multi-agent conversations.
\newblock In \emph{First Conference on Language Modeling}, 2024.
\newblock URL \url{https://openreview.net/forum?id=BAakY1hNKS}.

\bibitem[Cao et~al.(2024)Cao, Cai, Zhang, Zou, and Lam]{cao2024on}
Bowen Cao, Deng Cai, Zhisong Zhang, Yuexian Zou, and Wai Lam.
\newblock On the worst prompt performance of large language models.
\newblock In \emph{The Thirty-eighth Annual Conference on Neural Information Processing Systems}, 2024.
\newblock URL \url{https://openreview.net/forum?id=Mi853QaJx6}.

\bibitem[Lin et~al.(2021)Lin, Jaech, Li, Gormley, and Eisner]{lin-etal-2021-limitations}
Chu-Cheng Lin, Aaron Jaech, Xin Li, Matthew~R. Gormley, and Jason Eisner.
\newblock Limitations of autoregressive models and their alternatives.
\newblock In Kristina Toutanova, Anna Rumshisky, Luke Zettlemoyer, Dilek Hakkani-Tur, Iz~Beltagy, Steven Bethard, Ryan Cotterell, Tanmoy Chakraborty, and Yichao Zhou, editors, \emph{Proceedings of the 2021 Conference of the North American Chapter of the Association for Computational Linguistics: Human Language Technologies}, pages 5147--5173, Online, June 2021. Association for Computational Linguistics.
\newblock \doi{10.18653/v1/2021.naacl-main.405}.
\newblock URL \url{https://aclanthology.org/2021.naacl-main.405}.

\bibitem[Zhou et~al.(2023)Zhou, Muresanu, Han, Paster, Pitis, Chan, and Ba]{zhou2023large}
Yongchao Zhou, Andrei~Ioan Muresanu, Ziwen Han, Keiran Paster, Silviu Pitis, Harris Chan, and Jimmy Ba.
\newblock Large language models are human-level prompt engineers.
\newblock In \emph{The Eleventh International Conference on Learning Representations}, 2023.
\newblock URL \url{https://openreview.net/forum?id=92gvk82DE-}.

\bibitem[Pryzant et~al.(2023)Pryzant, Iter, Li, Lee, Zhu, and Zeng]{pryzant-etal-2023-automatic}
Reid Pryzant, Dan Iter, Jerry Li, Yin Lee, Chenguang Zhu, and Michael Zeng.
\newblock Automatic prompt optimization with ``gradient descent'' and beam search.
\newblock In Houda Bouamor, Juan Pino, and Kalika Bali, editors, \emph{Proceedings of the 2023 Conference on Empirical Methods in Natural Language Processing}, pages 7957--7968, Singapore, December 2023. Association for Computational Linguistics.
\newblock \doi{10.18653/v1/2023.emnlp-main.494}.
\newblock URL \url{https://aclanthology.org/2023.emnlp-main.494/}.

\bibitem[Yuksekgonul et~al.(2025)Yuksekgonul, Bianchi, Boen, Liu, Lu, Huang, Guestrin, and Zou]{yuksekgonul2025optimizing}
Mert Yuksekgonul, Federico Bianchi, Joseph Boen, Sheng Liu, Pan Lu, Zhi Huang, Carlos Guestrin, and James Zou.
\newblock Optimizing generative ai by backpropagating language model feedback.
\newblock \emph{Nature}, 639:\penalty0 609--616, 2025.

\bibitem[Jafari et~al.(2024)Jafari, Mekala, Yu, and Berg-Kirkpatrick]{jafari-etal-2024-morl}
Yasaman Jafari, Dheeraj Mekala, Rose Yu, and Taylor Berg-Kirkpatrick.
\newblock {MORL}-prompt: An empirical analysis of multi-objective reinforcement learning for discrete prompt optimization.
\newblock In Yaser Al-Onaizan, Mohit Bansal, and Yun-Nung Chen, editors, \emph{Findings of the Association for Computational Linguistics: EMNLP 2024}, pages 9878--9889, Miami, Florida, USA, November 2024. Association for Computational Linguistics.
\newblock \doi{10.18653/v1/2024.findings-emnlp.577}.
\newblock URL \url{https://aclanthology.org/2024.findings-emnlp.577/}.

\bibitem[Cohen et~al.(2012)Cohen, G{\'o}mez-Rodr{\'i}guez, and Satta]{Cohen2012EliminationOS}
Shay~B. Cohen, Carlos G{\'o}mez-Rodr{\'i}guez, and G.~Satta.
\newblock Elimination of spurious ambiguity in transition-based dependency parsing.
\newblock \emph{ArXiv}, abs/1206.6735, 2012.
\newblock URL \url{https://api.semanticscholar.org/CorpusID:15438603}.

\bibitem[Dohan et~al.(2022)Dohan, Xu, Lewkowycz, Austin, Bieber, Lopes, Wu, Michalewski, Saurous, Sohl-dickstein, Murphy, and Sutton]{dohan2022languagemodelcascades}
David Dohan, Winnie Xu, Aitor Lewkowycz, Jacob Austin, David Bieber, Raphael~Gontijo Lopes, Yuhuai Wu, Henryk Michalewski, Rif~A. Saurous, Jascha Sohl-dickstein, Kevin Murphy, and Charles Sutton.
\newblock Language model cascades, 2022.
\newblock URL \url{https://arxiv.org/abs/2207.10342}.

\bibitem[Wei and Tanner(1990)]{Wei1990AMC}
Greg C.~G. Wei and Martin~A. Tanner.
\newblock A monte carlo implementation of the em algorithm and the poor man's data augmentation algorithms.
\newblock \emph{Journal of the American Statistical Association}, 85:\penalty0 699--704, 1990.
\newblock URL \url{https://api.semanticscholar.org/CorpusID:123027134}.

\bibitem[Zelikman et~al.(2022)Zelikman, Wu, Mu, and Goodman]{zelikman2022}
Eric Zelikman, Yuhuai Wu, Jesse Mu, and Noah Goodman.
\newblock Star: Bootstrapping reasoning with reasoning.
\newblock In S.~Koyejo, S.~Mohamed, A.~Agarwal, D.~Belgrave, K.~Cho, and A.~Oh, editors, \emph{Advances in Neural Information Processing Systems}, volume~35, pages 15476--15488. Curran Associates, Inc., 2022.
\newblock URL \url{https://proceedings.neurips.cc/paper_files/paper/2022/file/639a9a172c044fbb64175b5fad42e9a5-Paper-Conference.pdf}.

\bibitem[Ritchie et~al.(2023)Ritchie, Guerrero, Jones, Mitra, Schulz, Willis, and Wu]{ritchie23}
Daniel Ritchie, Paul Guerrero, R.~Kenny Jones, Niloy~J. Mitra, Adriana Schulz, Karl D.~D. Willis, and Jiajun Wu.
\newblock Neurosymbolic models for computer graphics.
\newblock \emph{Computer Graphics Forum}, 42\penalty0 (2):\penalty0 545--568, 2023.
\newblock \doi{https://doi.org/10.1111/cgf.14775}.
\newblock URL \url{https://onlinelibrary.wiley.com/doi/abs/10.1111/cgf.14775}.

\bibitem[Goodfellow et~al.(2016)Goodfellow, Bengio, and Courville]{Goodfellow-et-al-2016-partition}
Ian Goodfellow, Yoshua Bengio, and Aaron Courville.
\newblock \emph{Deep Learning}, chapter~18.
\newblock MIT Press, 2016.
\newblock \url{http://www.deeplearningbook.org}.

\bibitem[Kingma and Welling(2014)]{Kingma2014}
Diederik~P. Kingma and Max Welling.
\newblock {Auto-Encoding Variational Bayes}.
\newblock In \emph{2nd International Conference on Learning Representations, {ICLR} 2014, Banff, AB, Canada, April 14-16, 2014, Conference Track Proceedings}, 2014.

\bibitem[Mnih and Gregor(2014)]{pmlr-v32-mnih14}
Andriy Mnih and Karol Gregor.
\newblock Neural variational inference and learning in belief networks.
\newblock In Eric~P. Xing and Tony Jebara, editors, \emph{Proceedings of the 31st International Conference on Machine Learning}, volume~32 of \emph{Proceedings of Machine Learning Research}, pages 1791--1799, Bejing, China, 22--24 Jun 2014. PMLR.
\newblock URL \url{https://proceedings.mlr.press/v32/mnih14.html}.

\bibitem[Veach and Guibas(1995)]{veach1995}
Eric Veach and Leonidas~J. Guibas.
\newblock Optimally combining sampling techniques for monte carlo rendering.
\newblock In \emph{Proceedings of the 22nd Annual Conference on Computer Graphics and Interactive Techniques}, SIGGRAPH '95, page 419–428, New York, NY, USA, 1995. Association for Computing Machinery.
\newblock ISBN 0897917014.
\newblock \doi{10.1145/218380.218498}.
\newblock URL \url{https://doi.org/10.1145/218380.218498}.

\bibitem[Bornschein and Bengio(2014)]{Bornschein2014ReweightedW}
J{\"o}rg Bornschein and Yoshua Bengio.
\newblock Reweighted wake-sleep.
\newblock \emph{CoRR}, abs/1406.2751, 2014.
\newblock URL \url{https://api.semanticscholar.org/CorpusID:10872458}.

\bibitem[Lin and Eisner(2018)]{lin-eisner-2018-neural}
Chu-Cheng Lin and Jason Eisner.
\newblock Neural particle smoothing for sampling from conditional sequence models.
\newblock In Marilyn Walker, Heng Ji, and Amanda Stent, editors, \emph{Proceedings of the 2018 Conference of the North {A}merican Chapter of the Association for Computational Linguistics: Human Language Technologies, Volume 1 (Long Papers)}, pages 929--941, New Orleans, Louisiana, June 2018. Association for Computational Linguistics.
\newblock \doi{10.18653/v1/N18-1085}.
\newblock URL \url{https://aclanthology.org/N18-1085/}.

\bibitem[Shi et~al.(2023)Shi, Suzgun, Freitag, Wang, Srivats, Vosoughi, Chung, Tay, Ruder, Zhou, Das, and Wei]{shi2022}
Freda Shi, Mirac Suzgun, Markus Freitag, Xuezhi Wang, Suraj Srivats, Soroush Vosoughi, Hyung~Won Chung, Yi~Tay, Sebastian Ruder, Denny Zhou, Dipanjan Das, and Jason Wei.
\newblock Language models are multilingual chain-of-thought reasoners.
\newblock In \emph{ICLR}, 2023.

\bibitem[Chen et~al.(2021)Chen, Chen, Smiley, Shah, Borova, Langdon, Moussa, Beane, Huang, Routledge, and Wang]{chen-etal-2021-finqa}
Zhiyu Chen, Wenhu Chen, Charese Smiley, Sameena Shah, Iana Borova, Dylan Langdon, Reema Moussa, Matt Beane, Ting-Hao Huang, Bryan Routledge, and William~Yang Wang.
\newblock {F}in{QA}: A dataset of numerical reasoning over financial data.
\newblock In Marie-Francine Moens, Xuanjing Huang, Lucia Specia, and Scott Wen-tau Yih, editors, \emph{Proceedings of the 2021 Conference on Empirical Methods in Natural Language Processing}, pages 3697--3711, Online and Punta Cana, Dominican Republic, November 2021. Association for Computational Linguistics.
\newblock \doi{10.18653/v1/2021.emnlp-main.300}.
\newblock URL \url{https://aclanthology.org/2021.emnlp-main.300}.

\bibitem[Sprague et~al.(2024{\natexlab{a}})Sprague, Ye, Bostrom, Chaudhuri, and Durrett]{sprague2024musr}
Zayne~Rea Sprague, Xi~Ye, Kaj Bostrom, Swarat Chaudhuri, and Greg Durrett.
\newblock Mu{SR}: Testing the limits of chain-of-thought with multistep soft reasoning.
\newblock In \emph{The Twelfth International Conference on Learning Representations}, 2024{\natexlab{a}}.
\newblock URL \url{https://openreview.net/forum?id=jenyYQzue1}.

\bibitem[Yang et~al.(2018)Yang, Qi, Zhang, Bengio, Cohen, Salakhutdinov, and Manning]{yang-etal-2018-hotpotqa}
Zhilin Yang, Peng Qi, Saizheng Zhang, Yoshua Bengio, William Cohen, Ruslan Salakhutdinov, and Christopher~D. Manning.
\newblock {H}otpot{QA}: A dataset for diverse, explainable multi-hop question answering.
\newblock In Ellen Riloff, David Chiang, Julia Hockenmaier, and Jun{'}ichi Tsujii, editors, \emph{Proceedings of the 2018 Conference on Empirical Methods in Natural Language Processing}, pages 2369--2380, Brussels, Belgium, October-November 2018. Association for Computational Linguistics.
\newblock \doi{10.18653/v1/D18-1259}.
\newblock URL \url{https://aclanthology.org/D18-1259/}.

\bibitem[Team et~al.(2024)Team, Riviere, Pathak, Sessa, Hardin, Bhupatiraju, Hussenot, Mesnard, Shahriari, Ramé, Ferret, Liu, Tafti, Friesen, Casbon, Ramos, Kumar, Lan, Jerome, Tsitsulin, Vieillard, Stanczyk, Girgin, Momchev, Hoffman, Thakoor, Grill, Neyshabur, Bachem, Walton, Severyn, Parrish, Ahmad, Hutchison, Abdagic, Carl, Shen, Brock, Coenen, Laforge, Paterson, Bastian, Piot, Wu, Royal, Chen, Kumar, Perry, Welty, Choquette-Choo, Sinopalnikov, Weinberger, Vijaykumar, Rogozińska, Herbison, Bandy, Wang, Noland, Moreira, Senter, Eltyshev, Visin, Rasskin, Wei, Cameron, Martins, Hashemi, Klimczak-Plucińska, Batra, Dhand, Nardini, Mein, Zhou, Svensson, Stanway, Chan, Zhou, Carrasqueira, Iljazi, Becker, Fernandez, van Amersfoort, Gordon, Lipschultz, Newlan, yeong Ji, Mohamed, Badola, Black, Millican, McDonell, Nguyen, Sodhia, Greene, Sjoesund, Usui, Sifre, Heuermann, Lago, McNealus, Soares, Kilpatrick, Dixon, Martins, Reid, Singh, Iverson, Görner, Velloso, Wirth, Davidow, Miller, Rahtz, Watson, Risdal,
  Kazemi, Moynihan, Zhang, Kahng, Park, Rahman, Khatwani, Dao, Bardoliwalla, Devanathan, Dumai, Chauhan, Wahltinez, Botarda, Barnes, Barham, Michel, Jin, Georgiev, Culliton, Kuppala, Comanescu, Merhej, Jana, Rokni, Agarwal, Mullins, Saadat, Carthy, Cogan, Perrin, Arnold, Krause, Dai, Garg, Sheth, Ronstrom, Chan, Jordan, Yu, Eccles, Hennigan, Kocisky, Doshi, Jain, Yadav, Meshram, Dharmadhikari, Barkley, Wei, Ye, Han, Kwon, Xu, Shen, Gong, Wei, Cotruta, Kirk, Rao, Giang, Peran, Warkentin, Collins, Barral, Ghahramani, Hadsell, Sculley, Banks, Dragan, Petrov, Vinyals, Dean, Hassabis, Kavukcuoglu, Farabet, Buchatskaya, Borgeaud, Fiedel, Joulin, Kenealy, Dadashi, and Andreev]{gemmateam2024gemma2improvingopen}
Gemma Team, Morgane Riviere, Shreya Pathak, Pier~Giuseppe Sessa, Cassidy Hardin, Surya Bhupatiraju, Léonard Hussenot, Thomas Mesnard, Bobak Shahriari, Alexandre Ramé, Johan Ferret, Peter Liu, Pouya Tafti, Abe Friesen, Michelle Casbon, Sabela Ramos, Ravin Kumar, Charline~Le Lan, Sammy Jerome, Anton Tsitsulin, Nino Vieillard, Piotr Stanczyk, Sertan Girgin, Nikola Momchev, Matt Hoffman, Shantanu Thakoor, Jean-Bastien Grill, Behnam Neyshabur, Olivier Bachem, Alanna Walton, Aliaksei Severyn, Alicia Parrish, Aliya Ahmad, Allen Hutchison, Alvin Abdagic, Amanda Carl, Amy Shen, Andy Brock, Andy Coenen, Anthony Laforge, Antonia Paterson, Ben Bastian, Bilal Piot, Bo~Wu, Brandon Royal, Charlie Chen, Chintu Kumar, Chris Perry, Chris Welty, Christopher~A. Choquette-Choo, Danila Sinopalnikov, David Weinberger, Dimple Vijaykumar, Dominika Rogozińska, Dustin Herbison, Elisa Bandy, Emma Wang, Eric Noland, Erica Moreira, Evan Senter, Evgenii Eltyshev, Francesco Visin, Gabriel Rasskin, Gary Wei, Glenn Cameron, Gus Martins,
  Hadi Hashemi, Hanna Klimczak-Plucińska, Harleen Batra, Harsh Dhand, Ivan Nardini, Jacinda Mein, Jack Zhou, James Svensson, Jeff Stanway, Jetha Chan, Jin~Peng Zhou, Joana Carrasqueira, Joana Iljazi, Jocelyn Becker, Joe Fernandez, Joost van Amersfoort, Josh Gordon, Josh Lipschultz, Josh Newlan, Ju~yeong Ji, Kareem Mohamed, Kartikeya Badola, Kat Black, Katie Millican, Keelin McDonell, Kelvin Nguyen, Kiranbir Sodhia, Kish Greene, Lars~Lowe Sjoesund, Lauren Usui, Laurent Sifre, Lena Heuermann, Leticia Lago, Lilly McNealus, Livio~Baldini Soares, Logan Kilpatrick, Lucas Dixon, Luciano Martins, Machel Reid, Manvinder Singh, Mark Iverson, Martin Görner, Mat Velloso, Mateo Wirth, Matt Davidow, Matt Miller, Matthew Rahtz, Matthew Watson, Meg Risdal, Mehran Kazemi, Michael Moynihan, Ming Zhang, Minsuk Kahng, Minwoo Park, Mofi Rahman, Mohit Khatwani, Natalie Dao, Nenshad Bardoliwalla, Nesh Devanathan, Neta Dumai, Nilay Chauhan, Oscar Wahltinez, Pankil Botarda, Parker Barnes, Paul Barham, Paul Michel, Pengchong Jin,
  Petko Georgiev, Phil Culliton, Pradeep Kuppala, Ramona Comanescu, Ramona Merhej, Reena Jana, Reza~Ardeshir Rokni, Rishabh Agarwal, Ryan Mullins, Samaneh Saadat, Sara~Mc Carthy, Sarah Cogan, Sarah Perrin, Sébastien M.~R. Arnold, Sebastian Krause, Shengyang Dai, Shruti Garg, Shruti Sheth, Sue Ronstrom, Susan Chan, Timothy Jordan, Ting Yu, Tom Eccles, Tom Hennigan, Tomas Kocisky, Tulsee Doshi, Vihan Jain, Vikas Yadav, Vilobh Meshram, Vishal Dharmadhikari, Warren Barkley, Wei Wei, Wenming Ye, Woohyun Han, Woosuk Kwon, Xiang Xu, Zhe Shen, Zhitao Gong, Zichuan Wei, Victor Cotruta, Phoebe Kirk, Anand Rao, Minh Giang, Ludovic Peran, Tris Warkentin, Eli Collins, Joelle Barral, Zoubin Ghahramani, Raia Hadsell, D.~Sculley, Jeanine Banks, Anca Dragan, Slav Petrov, Oriol Vinyals, Jeff Dean, Demis Hassabis, Koray Kavukcuoglu, Clement Farabet, Elena Buchatskaya, Sebastian Borgeaud, Noah Fiedel, Armand Joulin, Kathleen Kenealy, Robert Dadashi, and Alek Andreev.
\newblock Gemma 2: Improving open language models at a practical size, 2024.
\newblock URL \url{https://arxiv.org/abs/2408.00118}.

\bibitem[Yang et~al.(2025)Yang, Li, Yang, Zhang, Hui, Zheng, Yu, Gao, Huang, Lv, Zheng, Liu, Zhou, Huang, Hu, Ge, Wei, Lin, Tang, Yang, Tu, Zhang, Yang, Yang, Zhou, Zhou, Lin, Dang, Bao, Yang, Yu, Deng, Li, Xue, Li, Zhang, Wang, Zhu, Men, Gao, Liu, Luo, Li, Tang, Yin, Ren, Wang, Zhang, Ren, Fan, Su, Zhang, Zhang, Wan, Liu, Wang, Cui, Zhang, Zhou, and Qiu]{qwen3}
An~Yang, Anfeng Li, Baosong Yang, Beichen Zhang, Binyuan Hui, Bo~Zheng, Bowen Yu, Chang Gao, Chengen Huang, Chenxu Lv, Chujie Zheng, Dayiheng Liu, Fan Zhou, Fei Huang, Feng Hu, Hao Ge, Haoran Wei, Huan Lin, Jialong Tang, Jian Yang, Jianhong Tu, Jianwei Zhang, Jianxin Yang, Jiaxi Yang, Jing Zhou, Jingren Zhou, Junyang Lin, Kai Dang, Keqin Bao, Kexin Yang, Le~Yu, Lianghao Deng, Mei Li, Mingfeng Xue, Mingze Li, Pei Zhang, Peng Wang, Qin Zhu, Rui Men, Ruize Gao, Shixuan Liu, Shuang Luo, Tianhao Li, Tianyi Tang, Wenbiao Yin, Xingzhang Ren, Xinyu Wang, Xinyu Zhang, Xuancheng Ren, Yang Fan, Yang Su, Yichang Zhang, Yinger Zhang, Yu~Wan, Yuqiong Liu, Zekun Wang, Zeyu Cui, Zhenru Zhang, Zhipeng Zhou, and Zihan Qiu.
\newblock Qwen3 technical report.
\newblock \emph{arXiv preprint arXiv:2505.09388}, 2025.

\bibitem[Hu et~al.(2022)Hu, Shen, Wallis, Allen-Zhu, Li, Wang, Wang, and Chen]{hu2022lora}
Edward~J Hu, Yelong Shen, Phillip Wallis, Zeyuan Allen-Zhu, Yuanzhi Li, Shean Wang, Lu~Wang, and Weizhu Chen.
\newblock Lo{RA}: Low-rank adaptation of large language models.
\newblock In \emph{ICLR}, 2022.
\newblock URL \url{https://openreview.net/forum?id=nZeVKeeFYf9}.

\bibitem[Dong et~al.(2024)Dong, Ruan, Cai, Lai, Xu, Zhao, and Chen]{dong2024xgrammar}
Yixin Dong, Charlie~F Ruan, Yaxing Cai, Ruihang Lai, Ziyi Xu, Yilong Zhao, and Tianqi Chen.
\newblock Xgrammar: Flexible and efficient structured generation engine for large language models.
\newblock \emph{Proceedings of Machine Learning and Systems 7}, 2024.

\bibitem[Opsahl-Ong et~al.(2024)Opsahl-Ong, Ryan, Purtell, Broman, Potts, Zaharia, and Khattab]{opsahl-ong-etal-2024-optimizing}
Krista Opsahl-Ong, Michael~J Ryan, Josh Purtell, David Broman, Christopher Potts, Matei Zaharia, and Omar Khattab.
\newblock Optimizing instructions and demonstrations for multi-stage language model programs.
\newblock In Yaser Al-Onaizan, Mohit Bansal, and Yun-Nung Chen, editors, \emph{Proceedings of the 2024 Conference on Empirical Methods in Natural Language Processing}, pages 9340--9366, Miami, Florida, USA, November 2024. Association for Computational Linguistics.
\newblock \doi{10.18653/v1/2024.emnlp-main.525}.
\newblock URL \url{https://aclanthology.org/2024.emnlp-main.525/}.

\bibitem[Khattab et~al.(2024)Khattab, Singhvi, Maheshwari, Zhang, Santhanam, Vardhamanan, Haq, Sharma, Joshi, Moazam, Miller, Zaharia, and Potts]{khattab2024dspy}
Omar Khattab, Arnav Singhvi, Paridhi Maheshwari, Zhiyuan Zhang, Keshav Santhanam, Sri Vardhamanan, Saiful Haq, Ashutosh Sharma, Thomas~T. Joshi, Hanna Moazam, Heather Miller, Matei Zaharia, and Christopher Potts.
\newblock Dspy: Compiling declarative language model calls into self-improving pipelines.
\newblock 2024.

\bibitem[Beurer-Kellner et~al.(2023)Beurer-Kellner, Fischer, and Vechev]{lmql}
Luca Beurer-Kellner, Marc Fischer, and Martin Vechev.
\newblock Prompting is programming: A query language for large language models.
\newblock \emph{Proc. ACM Program. Lang.}, 7\penalty0 (PLDI), June 2023.
\newblock \doi{10.1145/3591300}.
\newblock URL \url{https://doi.org/10.1145/3591300}.

\bibitem[Soylu et~al.(2024)Soylu, Potts, and Khattab]{soylu-etal-2024-fine}
Dilara Soylu, Christopher Potts, and Omar Khattab.
\newblock Fine-tuning and prompt optimization: Two great steps that work better together.
\newblock In Yaser Al-Onaizan, Mohit Bansal, and Yun-Nung Chen, editors, \emph{Proceedings of the 2024 Conference on Empirical Methods in Natural Language Processing}, pages 10696--10710, Miami, Florida, USA, November 2024. Association for Computational Linguistics.
\newblock \doi{10.18653/v1/2024.emnlp-main.597}.
\newblock URL \url{https://aclanthology.org/2024.emnlp-main.597/}.

\bibitem[Tran et~al.(2017)Tran, Hoffman, Saurous, Brevdo, Murphy, and Blei]{tran2017deep}
Dustin Tran, Matthew~D. Hoffman, Rif~A. Saurous, Eugene Brevdo, Kevin Murphy, and David~M. Blei.
\newblock Deep probabilistic programming.
\newblock In \emph{International Conference on Learning Representations}, 2017.

\bibitem[Bingham et~al.(2019)Bingham, Chen, Jankowiak, Obermeyer, Pradhan, Karaletsos, Singh, Szerlip, Horsfall, and Goodman]{bingham2019pyro}
Eli Bingham, Jonathan~P. Chen, Martin Jankowiak, Fritz Obermeyer, Neeraj Pradhan, Theofanis Karaletsos, Rohit Singh, Paul~A. Szerlip, Paul Horsfall, and Noah~D. Goodman.
\newblock Pyro: Deep universal probabilistic programming.
\newblock \emph{J. Mach. Learn. Res.}, 20:\penalty0 28:1--28:6, 2019.
\newblock URL \url{http://jmlr.org/papers/v20/18-403.html}.

\bibitem[Lafferty et~al.(2001)Lafferty, McCallum, and Pereira]{crf}
John~D. Lafferty, Andrew McCallum, and Fernando C.~N. Pereira.
\newblock Conditional random fields: Probabilistic models for segmenting and labeling sequence data.
\newblock In \emph{Proceedings of the Eighteenth International Conference on Machine Learning}, ICML '01, page 282–289, San Francisco, CA, USA, 2001. Morgan Kaufmann Publishers Inc.
\newblock ISBN 1558607781.

\bibitem[Belanger and McCallum(2016)]{belanger}
David Belanger and Andrew McCallum.
\newblock Structured prediction energy networks.
\newblock In \emph{Proceedings of the 33rd International Conference on International Conference on Machine Learning - Volume 48}, ICML'16, page 983–992. JMLR.org, 2016.

\bibitem[Wei et~al.(2022{\natexlab{b}})Wei, Wang, Schuurmans, Bosma, hsin Chi, Xia, Le, and Zhou]{Wei2022ChainOT}
Jason Wei, Xuezhi Wang, Dale Schuurmans, Maarten Bosma, Ed~Huai hsin Chi, F.~Xia, Quoc Le, and Denny Zhou.
\newblock Chain of thought prompting elicits reasoning in large language models.
\newblock \emph{ArXiv}, abs/2201.11903, 2022{\natexlab{b}}.
\newblock URL \url{https://api.semanticscholar.org/CorpusID:246411621}.

\bibitem[Madaan et~al.(2023)Madaan, Tandon, Gupta, Hallinan, Gao, Wiegreffe, Alon, Dziri, Prabhumoye, Yang, Gupta, Majumder, Hermann, Welleck, Yazdanbakhsh, and Clark]{selfrefine}
Aman Madaan, Niket Tandon, Prakhar Gupta, Skyler Hallinan, Luyu Gao, Sarah Wiegreffe, Uri Alon, Nouha Dziri, Shrimai Prabhumoye, Yiming Yang, Shashank Gupta, Bodhisattwa~Prasad Majumder, Katherine Hermann, Sean Welleck, Amir Yazdanbakhsh, and Peter Clark.
\newblock Self-refine: Iterative refinement with self-feedback.
\newblock In A.~Oh, T.~Naumann, A.~Globerson, K.~Saenko, M.~Hardt, and S.~Levine, editors, \emph{Advances in Neural Information Processing Systems}, volume~36, pages 46534--46594. Curran Associates, Inc., 2023.
\newblock URL \url{https://proceedings.neurips.cc/paper_files/paper/2023/file/91edff07232fb1b55a505a9e9f6c0ff3-Paper-Conference.pdf}.

\bibitem[Trung et~al.(2024)Trung, Zhang, Jie, Sun, Jin, and Li]{trung-etal-2024-reft}
Luong Trung, Xinbo Zhang, Zhanming Jie, Peng Sun, Xiaoran Jin, and Hang Li.
\newblock {R}e{FT}: Reasoning with reinforced fine-tuning.
\newblock In Lun-Wei Ku, Andre Martins, and Vivek Srikumar, editors, \emph{Proceedings of the 62nd Annual Meeting of the Association for Computational Linguistics (Volume 1: Long Papers)}, pages 7601--7614, Bangkok, Thailand, August 2024. Association for Computational Linguistics.
\newblock \doi{10.18653/v1/2024.acl-long.410}.
\newblock URL \url{https://aclanthology.org/2024.acl-long.410/}.

\bibitem[Poesia et~al.(2022)Poesia, Polozov, Le, Tiwari, Soares, Meek, and Gulwani]{poesia2022synchromesh}
Gabriel Poesia, Alex Polozov, Vu~Le, Ashish Tiwari, Gustavo Soares, Christopher Meek, and Sumit Gulwani.
\newblock Synchromesh: Reliable code generation from pre-trained language models.
\newblock In \emph{International Conference on Learning Representations}, 2022.
\newblock URL \url{https://openreview.net/forum?id=KmtVD97J43e}.

\bibitem[Geng et~al.(2023)Geng, Josifoski, Peyrard, and West]{geng-etal-2023-grammar}
Saibo Geng, Martin Josifoski, Maxime Peyrard, and Robert West.
\newblock Grammar-constrained decoding for structured {NLP} tasks without finetuning.
\newblock In Houda Bouamor, Juan Pino, and Kalika Bali, editors, \emph{Proceedings of the 2023 Conference on Empirical Methods in Natural Language Processing}, pages 10932--10952, Singapore, December 2023. Association for Computational Linguistics.
\newblock \doi{10.18653/v1/2023.emnlp-main.674}.
\newblock URL \url{https://aclanthology.org/2023.emnlp-main.674/}.

\bibitem[McCarthy et~al.(2023)McCarthy, Zhang, Kumar, Stahlberg, and Wu]{mccarthy-etal-2023-long}
Arya McCarthy, Hao Zhang, Shankar Kumar, Felix Stahlberg, and Ke~Wu.
\newblock Long-form speech translation through segmentation with finite-state decoding constraints on large language models.
\newblock In Houda Bouamor, Juan Pino, and Kalika Bali, editors, \emph{Findings of the Association for Computational Linguistics: EMNLP 2023}, pages 247--257, Singapore, December 2023. Association for Computational Linguistics.
\newblock \doi{10.18653/v1/2023.findings-emnlp.19}.
\newblock URL \url{https://aclanthology.org/2023.findings-emnlp.19/}.

\bibitem[Beurer-Kellner et~al.(2024)Beurer-Kellner, Fischer, and Vechev]{10.5555/3692070.3692216}
Luca Beurer-Kellner, Marc Fischer, and Martin Vechev.
\newblock Guiding llms the right way: fast, non-invasive constrained generation.
\newblock In \emph{Proceedings of the 41st International Conference on Machine Learning}, ICML'24. JMLR.org, 2024.

\bibitem[Geng et~al.(2025)Geng, Cooper, Moskal, Jenkins, Berman, Ranchin, West, Horvitz, and Nori]{geng2025jsonschemabench}
Saibo Geng, Hudson Cooper, Michał Moskal, Samuel Jenkins, Julian Berman, Nathan Ranchin, Robert West, Eric Horvitz, and Harsha Nori.
\newblock Generating structured outputs from language models: Benchmark and studies, 2025.
\newblock URL \url{https://arxiv.org/abs/2501.10868}.

\bibitem[Houlsby et~al.(2019)Houlsby, Giurgiu, Jastrzebski, Morrone, De~Laroussilhe, Gesmundo, Attariyan, and Gelly]{pmlr-v97-houlsby19a}
Neil Houlsby, Andrei Giurgiu, Stanislaw Jastrzebski, Bruna Morrone, Quentin De~Laroussilhe, Andrea Gesmundo, Mona Attariyan, and Sylvain Gelly.
\newblock Parameter-efficient transfer learning for {NLP}.
\newblock In Kamalika Chaudhuri and Ruslan Salakhutdinov, editors, \emph{Proceedings of the 36th International Conference on Machine Learning}, volume~97 of \emph{Proceedings of Machine Learning Research}, pages 2790--2799. PMLR, 09--15 Jun 2019.
\newblock URL \url{https://proceedings.mlr.press/v97/houlsby19a.html}.

\bibitem[Li and Liang(2021)]{li-liang-2021-prefix}
Xiang~Lisa Li and Percy Liang.
\newblock Prefix-tuning: Optimizing continuous prompts for generation.
\newblock In Chengqing Zong, Fei Xia, Wenjie Li, and Roberto Navigli, editors, \emph{Proceedings of the 59th Annual Meeting of the Association for Computational Linguistics and the 11th International Joint Conference on Natural Language Processing (Volume 1: Long Papers)}, pages 4582--4597, Online, August 2021. Association for Computational Linguistics.
\newblock \doi{10.18653/v1/2021.acl-long.353}.
\newblock URL \url{https://aclanthology.org/2021.acl-long.353/}.

\bibitem[Lester et~al.(2021)Lester, Al-Rfou, and Constant]{lester-etal-2021-power}
Brian Lester, Rami Al-Rfou, and Noah Constant.
\newblock The power of scale for parameter-efficient prompt tuning.
\newblock In Marie-Francine Moens, Xuanjing Huang, Lucia Specia, and Scott Wen-tau Yih, editors, \emph{Proceedings of the 2021 Conference on Empirical Methods in Natural Language Processing}, pages 3045--3059, Online and Punta Cana, Dominican Republic, November 2021. Association for Computational Linguistics.
\newblock \doi{10.18653/v1/2021.emnlp-main.243}.
\newblock URL \url{https://aclanthology.org/2021.emnlp-main.243/}.

\bibitem[Liu et~al.(2022)Liu, Ji, Fu, Tam, Du, Yang, and Tang]{liu-etal-2022-p}
Xiao Liu, Kaixuan Ji, Yicheng Fu, Weng Tam, Zhengxiao Du, Zhilin Yang, and Jie Tang.
\newblock {P}-tuning: Prompt tuning can be comparable to fine-tuning across scales and tasks.
\newblock In Smaranda Muresan, Preslav Nakov, and Aline Villavicencio, editors, \emph{Proceedings of the 60th Annual Meeting of the Association for Computational Linguistics (Volume 2: Short Papers)}, pages 61--68, Dublin, Ireland, May 2022. Association for Computational Linguistics.
\newblock \doi{10.18653/v1/2022.acl-short.8}.
\newblock URL \url{https://aclanthology.org/2022.acl-short.8/}.

\bibitem[Williams(1992)]{Williams92}
Ronald~J. Williams.
\newblock Simple statistical gradient-following algorithms for connectionist reinforcement learning.
\newblock \emph{Mach. Learn.}, 8\penalty0 (3–4):\penalty0 229–256, May 1992.
\newblock ISSN 0885-6125.
\newblock \doi{10.1007/BF00992696}.
\newblock URL \url{https://doi.org/10.1007/BF00992696}.

\bibitem[Schulman et~al.(2017)Schulman, Wolski, Dhariwal, Radford, and Klimov]{schulman2017proximalpolicyoptimizationalgorithms}
John Schulman, Filip Wolski, Prafulla Dhariwal, Alec Radford, and Oleg Klimov.
\newblock Proximal policy optimization algorithms, 2017.
\newblock URL \url{https://arxiv.org/abs/1707.06347}.

\bibitem[Konda and Tsitsiklis(1999)]{NIPS1999_6449f44a}
Vijay Konda and John Tsitsiklis.
\newblock Actor-critic algorithms.
\newblock In S.~Solla, T.~Leen, and K.~M\"{u}ller, editors, \emph{Advances in Neural Information Processing Systems}, volume~12. MIT Press, 1999.
\newblock URL \url{https://proceedings.neurips.cc/paper_files/paper/1999/file/6449f44a102fde848669bdd9eb6b76fa-Paper.pdf}.

\bibitem[Sprague et~al.(2024{\natexlab{b}})Sprague, Yin, Rodriguez, Jiang, Wadhwa, Singhal, Zhao, Ye, Mahowald, and Durrett]{sprague2024cotcotchainofthoughthelps}
Zayne Sprague, Fangcong Yin, Juan~Diego Rodriguez, Dongwei Jiang, Manya Wadhwa, Prasann Singhal, Xinyu Zhao, Xi~Ye, Kyle Mahowald, and Greg Durrett.
\newblock To cot or not to cot? chain-of-thought helps mainly on math and symbolic reasoning, 2024{\natexlab{b}}.
\newblock URL \url{https://arxiv.org/abs/2409.12183}.

\bibitem[Kingma and Ba(2014)]{kingma2014adam}
Diederik~P Kingma and Jimmy Ba.
\newblock Adam: A method for stochastic optimization.
\newblock \emph{arXiv preprint arXiv:1412.6980}, 2014.

\bibitem[Holtzman et~al.(2020)Holtzman, Buys, Du, Forbes, and Choi]{nucleussampling}
Ari Holtzman, Jan Buys, Li~Du, Maxwell Forbes, and Yejin Choi.
\newblock The curious case of neural text degeneration.
\newblock In \emph{ICLR}. OpenReview.net, 2020.
\newblock URL \url{http://dblp.uni-trier.de/db/conf/iclr/iclr2020.html#HoltzmanBDFC20}.

\end{thebibliography}
\clearpage
\appendix
\section*{Appendices}

\begin{table*}[ht]
    \centering
    \begin{tabular}{p{.48\linewidth}p{.48\linewidth}}
    \toprule
      Program View & Probabilistic View \\
        \midrule
         $\tau$-typed object & Random variable $\in \Sigma^*$ restricted to strings $\in \mathrm{valid}(\tau)$ \\
         LM adaptor with weights $\vtheta$, with output restricted to $\tau$-typed objects & Unnormalized conditional distribution $p_{LM}(\mathbf{z}_{t} \mid \mathbf{z}_s ; \vtheta) \mathbb{I}(\mathbf{z}_{t} \in \mathrm{valid}(\tau))$ \\
         Deterministic algorithm $f: \tau_i \rightarrow \tau_o$ & Degenerate distribution $\delta_{\texttt{canon}(f(\texttt{parse}(x, \tau_i)))}(y)$ \\
         \texttt{parse} and \texttt{canon} functions that convert typed objects to/from LM inputs/outputs & Measurable maps between object domain $\mathcal{O}$ and string domain $\Sigma^*$ \\
         Executing a workflow to obtain $\mathbf{z}_{1 \ldots M}$ & Sampling from joint unnormalized probability $\tilde{p}_{\vtheta}(\mathbf{z}_{1 \ldots M}) = \prod_{k} \tilde{p}_{\vtheta}(\mathbf{z}_{T_k} \mid \mathbf{z}_{S_k}) $\\
         Probability that a stochastic workflow succeeds & $\mathcal{Z}_{\vtheta} = \mathrm{Pr}_{p_{\vtheta}}(\text{all nodes are valid})$ \\
        \bottomrule
    \end{tabular}
    \caption{Dual semantics: how \model\,concepts map between their program and probabilistic views.}
    \label{tab:dual-semantics}
\end{table*}

\section{Background and Related Work}
\label{sec:problem}

\paragraph{Programmatic LM workflows.}
A large body of work exposes LMs through \emph{programmed} pipelines as typed or templated modules, with declarative constraints and optimizers, such as DSPy \citep{khattab2022demonstrate,khattab2024dspy}, LMQL \citep{lmql}, and LangChain \citep{Chase_LangChain_2022}. These systems typically \emph{specify} structure and then tune prompts or few-shot exemplars. They do not cast the entire workfrlow as a single probabilistic object with learnable continuous parameters, and a likelihood objective. While there have been proposals that optimized weights under such programmatic pipelines (such as \texttt{BetterTogether} \citep{soylu-etal-2024-fine}),
\model s differs fundamentally in its principled yet optimization-friendly probabilistic formulation, which enables both theoretically justified training methods (\cref{sec:tac-star}) and advanced inference techniques (\cref{sec:amortized_model_star}).

\paragraph{Probabilistic programming and structured prediction.} Probabilistic programming languages tailored for machine learning, such as Edward \citep{tran2017deep} and Pyro \citep{bingham2019pyro}, combine differentiable components with stochastic control flow. On the other hand, classical structured prediction \citep{crf,belanger} provides tools for handling global constraints in unnormalized models. Our formulation connects these threads to LM workflows: each typed hyperedge is an \emph{unnormalized} conditional whose {\bf type compliance} functions as a partition function term $\mathcal{Z}_{\vtheta} \leq 1$. This distinct perspective allows us to train with a tractable objective, whose bias vanishes as type compliance rises.

\paragraph{Problem-solving strategies and adapting for reasoning.}
Techniques like Chain-of-Thought (CoT) \citep{Wei2022ChainOT} and Self-Refine \citep{selfrefine} leverage prompting to elicit intermediate problem-solving steps or iterative improvements from LMs, often boosting performance on complex tasks. Methods such as STaR \citep{zelikman2022} and ReFT
\citep{trung-etal-2024-reft} further adapted the LM to reason. We adopt the spirit of STaR, but place it inside a hypergraph to propose typed and multi-step rationalizations (\cref{sec:tac-star}). We also introduce an amortized variant that learns to propose rationalizations, rather than relying solely on heuristics (\cref{sec:amortized_model_star}).

\paragraph{Constrained and schema-aware decoding.} To improve output reliability, various methods enforce grammar-based constraints during LLM generation \citep{poesia2022synchromesh,geng-etal-2023-grammar,mccarthy-etal-2023-long,10.5555/3692070.3692216,geng2025jsonschemabench} have been proposed. These methods generally modify \emph{local} conditional distributions over next tokens, to mask out continuations that are incompatible with the given input and grammar. In contrast, our objective learns parameters so that type-compliant trajectories carry increasing probability mass globally, improving validity and task accuracy.

\paragraph{Parameter-efficient adaptation.}
LoRA and related PEFT methods \citep{pmlr-v97-houlsby19a,hu2022lora,li-liang-2021-prefix,lester-etal-2021-power,liu-etal-2022-p} enable light-weight adaptation. We use small adaptors to highlight data-efficiency and show that gains stem from \emph{typed workflow learning} rather than sheer capacity.

\paragraph{Connection to Reinforcement Learning.} The \model STaR training procedure (\cref{sec:tac-star}) can be viewed through the lens of policy optimization. As \citet{zelikman2022} observed, the STaR objective closely resembles the REINFORCE algorithm \citep{Williams92}. Similarly, the M-step in the \model STaR algorithm can be interpreted as optimizing the \model\,workflow policy under REINFORCE, where a binary reward is assigned upon successfully generating the correct output.

We adopt the MC-EM framing as it provides a principled approach for likelihood maximization in the presence of annotated output data. While more advanced RL techniques (\emph{e.g.}, PPO \citep{schulman2017proximalpolicyoptimizationalgorithms} or actor-critic methods \citep{NIPS1999_6449f44a}) work with non-binary reward functions, they often introduce complexity, such as training value functions, which are difficult to estimate over complex, typed latent spaces. Furthermore, the exploration challenge often faced by policy gradient methods in sparse reward settings is significantly mitigated by both the rationalization heuristic and the inference network in Amortized \model STaR (\cref{sec:amortized_model_star}) in the E-step. This mechanism effectively guides the sampling process towards successful trajectories using the known outputs \Dash a technique specific to this supervised adaptation context.

\section{Additional Studies on Workflow Pattern Design}

\label{sec:additional-studies-pattern-design}

\begin{figure}[h!]
    \centering
\subfloat[{\bf direct}]{\centering
\begin{tikzpicture}[>=Stealth]

    \node[observed, align=center] (x) {$\mathbf{z}_1$ \\ type: $\tau_i$};
    \node[observed, align=center, below=2cm of x] (y) {$\mathbf{z}_2$ \\ type: $\tau_o$};

    \draw[->] (x) edge node[above] {$(\tau_i, \tau_o, \vtheta_1)$}  (y);

\end{tikzpicture}
\label{fig:single-adaptor-tac}} \quad
    \subfloat[\centering {\bf cot-type-structure}]{
    \label{fig:cot-type-structure}
    {
    \begin{tikzpicture}[>=Stealth]

    \node[observed, align=center] (x) {$\mathbf{z}_1$ \\ type: $\tau_i$};
    \node[latent, align=center, below=2cm of x] (oprime) {$\mathbf{z}_3$ \\ type: $\tau_{ro}$};
    \node[observed, align=center, below=2cm of oprime] (y) {$\mathbf{z}_2$ \\ type: $\tau_o$};

    \draw[->] (x) edge node[above] {$(\tau_i, \tau_{ro}, \vtheta_2)$}  (oprime);
    \draw[->] (oprime) edge node[above] {$\mathrm{extract}$: $\tau_{ro} \rightarrow \tau_o$} (y);
\end{tikzpicture}
    }
    }%
\caption{
    Workflow patterns experimented in this paper, with increasing structural complexity from left to right.  In the most complicated pattern {\bf expression-cascade-structure} we illustrate the workflow with example node values. Dashed-boundary nodes indicate variables that are not observed at training time. And solid-boundary nodes indicate nodes with training time observable values.
    A main message of this work is that instead of defining workflows imperatively as fixed-parameter systems, {\bf we treat an entire typed workflow as a single probabilistic program, whose parameters are lightweight PEFT modules, allowing end-to-end training with latent variables.} }
    \label{fig:tac-patterns-2}
\end{figure}
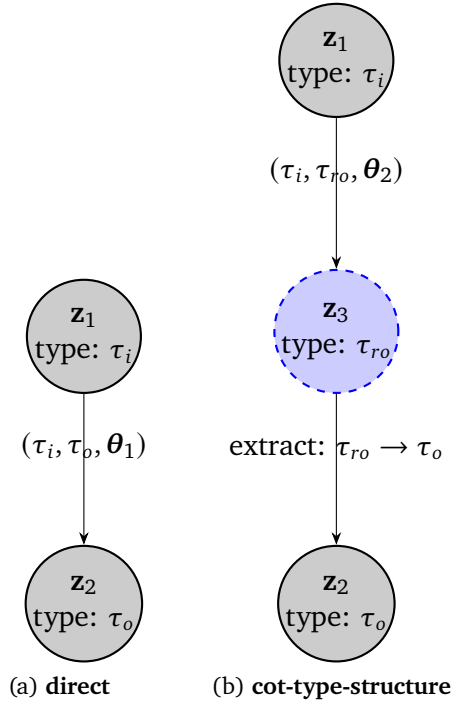

In this section, we conduct additional experiments that vary the pattern structures, and evaluate how such changes affect performance. Specifically, we would like to answer the following questions:
\begin{itemize}[noitemsep,topsep=0pt]
    \item {\bf (\cref{sec:effectiveness-of-latent-structures}) Is adaptation with reasoning workflows effective?} The \model\,framework gives practitioners great freedom in designing a workflow that reason in the process. We hypothesize that adapting with such explicit structures improves performance on tasks that require complex reasoning.
    \item {\bf (\cref{sec:does-cascade-design-matter}) How do \model\,design variations affect performance?}
    We evaluate how such \model\,design variations for the same task affect performance.
\end{itemize}

\subsection{End-to-end trainable workflows as \model s.}

The declarative and flexible nature of \model s enable practitioners to rapidly implement end-to-end trainable workflows. We implement some common patterns as \model s:
\begin{itemize}[noitemsep,topsep=0pt]
    \item {\bf Direct adaptation} of an LM to the downstream task without any latent structure corresponds to common supervised PEFT methods surveyed in \cref{sec:problem}. The {\bf direct} pattern (\cref{fig:single-adaptor-tac}) is a singleton \model\,with no latent nodes.
    \item {\bf Adapting with latent rationales} corresponds to patterns that learn to generate rationales for the task at hand \cite{zelikman2022}. There are several possible \model\,structure designs that incorporate rationales: for example, {\bf cot-type-structure} (\cref{fig:cot-type-structure}) maps the input to a rationale-output typed object, from which the task output is deterministically extracted. Alternatively, {\bf cot-cascade-structure} (\cref{fig:cot-cascade-structure}) introduce rationales as distinct nodes in the \model\,hypergraph, which transforms into the task output under an adaptor.
    \item {\bf Trainable self-refinement} refers to an end-to-end trainable variant of self-refine \citep{selfrefine}, where the model first sketches a task output, and iteratively refine it. Without \model , a practitioner would have to resort to manually writing tedious postprocessing functions for the intermediate results. On the other hand, the  \model\,counterpart {\bf refine-structure} (\cref{fig:multiple-adaptor-tac-refinement} in \cref{sec:model-figures}) is straightforward.
\end{itemize}

For the MGSM-SymPy task, we experiment with the {\bf expression-cascade-structure} pattern (\cref{fig:cascade}), which additionally imposes the constraint that the output must be a rational number represented by an arithmetic expression tree. Such type constraints often reflect business logic (for example, we expect the MGSM dataset to have rational number answers), and may be necessary when the \model\,forms a component in a larger system. 

\subsection{Effectiveness of Adaptation with Reasoning Workflows}
\label{sec:effectiveness-of-latent-structures}

To evaluate whether adaptation with reasoning workflows is effective, we compare {\bf cot-cascade-structure}, and {\bf refine-structure} \model s against {\bf direct} on the $3$ tasks MGSM, FinQA and HotPotQA, on base models \texttt{gemma-2-27b-it} and \texttt{gemma-1.1-7b-it}. \cref{tab:no-structure-cot-comparison} shows that both {\bf cot-cascade-structure}  significantly outperforms {\bf direct} on MGSM and FinQA on both \texttt{gemma-2-27b-it} and \texttt{gemma-1.1-7b-it}. But {\bf cot-cascade-structure} slightly underperforms {\bf direct} on HotPotQA. These results largely agree with the meta study done by \cite{sprague2024cotcotchainofthoughthelps}, which also reported that tasks that require arithmetic and symbolic reasoning, such as MGSM and FinQA, benefit the most from CoT, while a huge portion of previous work saw that CoT degrades performance for multihop QA. However, we note that the {\bf refine-structure} \model\, (\cref{fig:multiple-adaptor-tac-refinement}) consistently outperform the {\bf direct} baseline in all $3$ tasks on \texttt{gemma-2-27b-it}, showcasing the effectiveness of the adaptive refinement paradigm.

\begin{table*}[h]
\centering
{
\begin{tabular}{lrrrrrr}
\toprule
 & \multicolumn{3}{c}{\texttt{gemma-2-27b-it}} & \multicolumn{2}{c}{\texttt{gemma-1.1-7b-it}} \\
 \cmidrule(lr){2-4} \cmidrule(lr){5-6}
Dataset  & direct             & cot-cascade-structure & refine-structure & direct             & cot-cascade-structure \\
\midrule
MGSM     & {24.7} & {\bf 82.2} & 78.6 & 5.1 & {\bf 27.3}  \\
FinQA    & 17.3                     & {\bf 36.0}  & 23.7 & 3.0 &   {\bf 9.7}       \\
HotPotQA & 34.0                       & 32.0 & {\bf 39.0}        & \Dash & \Dash              \\
\bottomrule
\end{tabular}
}
\caption{Comparison between {\bf direct} and reasoning workflows. For the MGSM dataset, we report per-language accuracies in \cref{tab:mgsm-numbers}. The difference between best performing runs and {\bf direct} are statistically significant/marginally significant: for MGSM and FinQA $p < 0.05$ (both \texttt{gemma-2-27b-it} and \texttt{gemma-1.1-7b-it}), and for HotPotQA $p = 0.07$ under paired permutation tests. Per-language accuracy numbers of the MGSM dataset are in \cref{sec:per-language-mgsm-results}.}
\label{tab:no-structure-cot-comparison}
\end{table*}

\paragraph{Task adaptation with \model STaR is effective.} To evaluate whether the efficacy of \model s can be attributed to our proposed \model STaR method, we also compare adapted \model\,workflows against those with the same hypergraph structure, but with un-adapted weights (\emph{i.e.}, all adaptors in the \model\,use base model weights). Both \model STaR trained and un-adapted models use the same structured LangFun prompts that are similar to examples listed in \cref{sec:prompts-generated-by-langfun}.
The significant gap between adapted and un-adapted results in \cref{tab:no-train-comparison} indicate that the \model STaR algorithm is effective. Notably, un-adapted models still outperform {\bf direct} workflows (listed in \cref{tab:no-structure-cot-comparison}), indicating that LangFun's type-inducing prompts can invoke somewhat effective test-time computation over the \model\,hypergraph structure.

\begin{table}[h]
\centering
{
\begin{tabular}{llrrr}
\toprule
Task   & Structure & \model STaR & Un-adapted \\
\midrule
MGSM     & {\small cot-cascade-structure} & {82.2} & {45.4}  \\
MGSM     & {\small cot-type-structure} & {80.4} & {74.7}  \\
MGSM-SymPy     & {\small expression-cascade-structure} & {75.9} & {69.5}  \\
FinQA & {\small cot-cascade-structure} & {36.0} & {13.0}  \\
HotPotQA & {\small refine-structure} & {39.0} & {24.0}  \\
\bottomrule
\end{tabular}
}
\caption{Comparison between \model STaR-adapted and un-adapted \texttt{gemma-2-27b-it}. The differences are all statistically significant ($p < 0.05$) under paired permutation tests.}
\label{tab:no-train-comparison}
\end{table}

\subsection{Effects of Different \model\,Designs}

\label{sec:does-cascade-design-matter}

\paragraph{Decoupling rationale and output modeling helps performance.} {\bf cot-cascade-structure} (\cref{fig:cot-cascade-structure}) achieves a higher score than {\bf cot-type-structure} (\cref{fig:cot-type-structure}) on the MGSM task (\cref{tab:engineering-comparison}), suggesting that modeling the rationale and task output generation with distinct adaptors helps performance.
By using distinct adaptors, the workflow allows specialization: the first adaptor focuses on reasoning, while the second specializes in synthesis, reducing the complexity burden on a single monolithic step.
The positive result again highlights how the \model\,formalism can help practitioners iterate and experiment with different multi-adaptor cascade designs, which would be tedious otherwise.

\paragraph{Robustness to Semantic Constraints.} Comparing performance on MGSM and the more constrained MGSM-SymPy task reveals a key advantage of the \model\, framework's robustness. As shown in \cref{tab:engineering-comparison}, the best-performing \model\, model sees a modest performance drop, from $82.2\%$ on MGSM to $75.9\%$ on MGSM-SymPy, when required to generate a valid symbolic expression.\footnote{Sample expressions generated under {\bf expression-cascade-structure} are listed in \cref{sec:expressions-from-ecs}.} This contrasts sharply with the prompt-optimizing baseline (\cref{tbl:prompt-optimization-comparison}). The best DSPy configuration experiences a much more significant degradation, plummeting from $81.9\%$ on MGSM to just $57.1\%$ on MGSM-SymPy. The substantially smaller performance drop for \model s underscores the brittleness of discrete prompt optimization when faced with strict structural requirements. The \model\, framework's gradient-based adaptation within a typed system proves to be significantly more resilient, making it a more reliable paradigm for tasks demanding structural compliance.

\begin{table}[h]
\centering
{
\begin{tabular}{p{.2\linewidth}p{.3\linewidth}p{.3\linewidth}}
\toprule
\multicolumn{2}{c}{MGSM} & {MGSM-SymPy} \\
\cmidrule(lr){1-2} \cmidrule(lr){3-3}
 {\bf cot-type-structure} & {\bf cot-cascade-structure} & {\bf expression-cascade-structure} \\
\midrule
{80.4} & {82.2} & {75.9} \\
\bottomrule
\end{tabular}
}
\caption{Effects of different \model\, designs on the MGSM dataset, demonstrating the impact of workflow structure on performance. The {\bf cot-cascade-structure} (which decouples rationale generation from the final answer synthesis) outperforms the monolithic {\bf cot-type-structure}. The {\bf expression-cascade-structure} result shows strong performance on the more constrained MGSM-SymPy task.}
\label{tab:engineering-comparison}
\end{table}

\section{Algorithms}
\label{sec:algorithms}

\subsection{Forward and Backward}

\Cref{alg:cascade-forward} (\texttt{forward}) executes the probabilistic program represented by a \model\,$C = (\mathbf{Z}, \mathbf{E})$. Starting from a given input node value $\mathbf{z}_1^*$, the algorithm traverses the hypergraph following a topological order, and terminates when all edges $\in \mathbf{Z}$ have been visited. \texttt{forward} takes $C$ and $\mathbf{z}^*_1$ as input arguments. \texttt{forward} also takes the following as arguments:
\begin{itemize}
    \item sampler configuration $\kappa$ for different sampling techniques, \emph{e.g.}, varying temperature, nucleus, and top-$k$ sampling
    \item maximum number of sampling attempts
\end{itemize}

        \begin{algorithm}[h]
            \caption{\model\, Forward Algorithm (\texttt{forward})}
            \label{alg:cascade-forward}
            \begin{algorithmic}[1]
                \Require  \model\, cascade $C = (\mathbf{Z}, \mathbf{E})$ where $\mathbf{Z} = \{ \mathbf{z}_1 \ldots \mathbf{z}_M \}$ and $\mathbf{E} = \{ \mathbf{e}_1 \ldots \mathbf{e}_K \}$, input object: $\mathbf{z}^*_1$, sampler configuration $\kappa$, $N_{\text{max}}$ for maximum number of sampling attempts.
                \Ensure Sampled values $(\mathbf{z}_2^*, \ldots, \mathbf{z}_M^*)$.
                \State Determine a topological ordering of edges in $\mathbf{E}$. Let the sorted hyperedges be ${e}'_1 \ldots {e}'_K$.
                \State $\mathbf{Z}^*_{\text{already\_sampled}} \leftarrow \{ \mathbf{z}_1^* \}$.
                \For{$k \in [1 .. K]$}
                    \State Assert the source nodes of $e'_k$ is a subset of $\mathbf{z}_{\text{already\_sampled}}$.
                    \If{$e'_k = (\tau_i, \tau_o, \vtheta)$ is a type-constrained LM adaptor}
                        \State {\# type-constrained LM adaptors have a single source node and a single target node.}
                        \State $\vx \leftarrow$ canonicalized representation of $e'_k$'s source node.
                        \While{number of attempts $\leq N_{\text{max}}$}
                        \State Try draw $\vy \sim p_{LM}(\cdot \mid \vx ; \vtheta, \kappa)$
                        \If{$\texttt{parse}(\vy, \tau_o) \neq \text{error}$}
                        \State $t \leftarrow$ index of $e'_k$'s target node.
                        \State $\mathbf{z}^*_t \leftarrow \vy$
                        \State $\mathbf{Z}^*_{\text{already\_sampled}} \leftarrow \mathbf{Z}_{\text{already\_sampled}} \cup \{ \mathbf{z}^*_t \}$
                        \State break
                        \EndIf
                        \EndWhile
                    \ElsIf{$e'_k$ is a deterministic algorithm $f$}
                        \State {\bf \# In this work we assume $f$'s inputs and outputs are sorted by node index in $C$.}
                        \State $\mathbf{O}_{f \text{input}} \leftarrow$ parsed objects of $e_k'$'s source nodes, sorted by node index.
                        \State $\mathbf{O}_{f \text{output}} \leftarrow f(\mathbf{O}_{f \text{input}})$
                        \State $\mathbf{Z}^*_{f \text{output}} \leftarrow $ canonicalized representations of objects $\in \mathbf{O}_{f \text{output}}$, sorted by node index.
                        \State $\mathbf{Z}_{\text{already\_sampled}} \leftarrow \mathbf{Z}^*_{\text{already\_sampled}} \cup \mathbf{Z}^*_{f \text{output}}$
                    \EndIf
                \EndFor
                \State \Return $\mathbf{Z}_{\text{already\_sampled}} - \{ \mathbf{z}^*_1 \}$.
            \end{algorithmic}
        \end{algorithm}

\Cref{alg:cascade-backward} (\texttt{backward}) takes as input $(C, \mathbf{Z}^*)$, where $C = (\mathbf{Z}, \mathbf{E})$ where $\mathbf{E} = (e_1 \ldots e_K)$ is a \model , and $\mathbf{Z}^*$ are value assignments of $\mathbf{Z}$. We assume the $\log$ probability $p_{LM}(\vy \mid \vx ; \vtheta_{k})$ is auto-differentiable with regard to all adaptor hyperedges in a \model . \Cref{alg:cascade-backward} returns unnormalized $\log$ joint probabilities of $\mathbf{Z}^*$ under $C$: $\log \tilde{p}_{\vtheta}(\mathbf{Z}^*)$, the per-node generation $\log$ probabilities $(\log p_{\vtheta}(z_2 \mid \cdot) \ldots \log p_{\vtheta}(z_M \mid \cdot))$, and also gradients of LM adaptors: $\nabla_{\vtheta_k} \log \tilde{p}_{\vtheta}(\mathbf{Z}^*)$ for adaptor hyperedges' indices $k$. We note that \texttt{backward} is easily parallelizable: all adaptor edges can be processed at the same time.

        \begin{algorithm}[h]
            \caption{\model\,Backward Algorithm (\texttt{backward})}
            \label{alg:cascade-backward}
            \begin{algorithmic}[1]
                \Require $C = (\mathbf{Z}, \mathbf{E})$ and sample $\mathbf{Z}^* = \{\mathbf{z}_1^*, \mathbf{z}_2^*, \ldots, \mathbf{z}^*_{M}\}$
                \Ensure $(\log \tilde{p}_{\vtheta}(\mathbf{Z}^*)$, $(\log p_{\vtheta}(z_2 \mid \cdot) \ldots \log p_{\vtheta}(z_M \mid \cdot))$, $\{\nabla_{\vtheta_k} \log \tilde{p}_{\vtheta}(\mathbf{Z}^*) \mid e_k \in \mathbf{E} \text{ is an adaptor hyperedge} \})$
                \State Initialize $\log$-probability accumulator $\mathcal{L} \leftarrow 0$.
                \For{each LM adaptor hyperedge $e_k = (\tau_i, \tau_o, \vtheta_k)$}
                    \State Let $\mathbf{z}^*_i \in \mathbf{Z}^*$, $\mathbf{z}^*_o \in \mathbf{Z}^*$ be the sample value of $e_k$'s input and output nodes $\mathbf{z}_i$ (typed $\tau_i$) and $\mathbf{z}_o$ respectively.
                    \State $(\ell, \mathbf{g}_k) \leftarrow \texttt{peft\_backward}(\log p_{LM}(\mathbf{z}^*_o \mid \texttt{canon}(\texttt{parse}(\mathbf{z}^*_i, \tau_i)); \vtheta )$.
                    \State $\mathcal{L} \leftarrow \mathcal{L} + \ell$
                    \State keep track of $\ell$ by its node index.
                \EndFor
                \State {\bf \# For nodes from deterministic hyperedges, set log prob to $0$ as they have no learnable parameters.}
                \State \Return $\left(\mathcal{L}, (\log p_{\vtheta}(z_2 \mid \cdot) \ldots \log p_{\vtheta}(z_M \mid \cdot)), \{\mathbf{g}_k \mid e_k \in \mathbf{E} \text{ is an adaptor hyperedge} \} \right)$.
            \end{algorithmic}
        \end{algorithm}

\subsection{\model STaR}

\label{sec:appendix-tac-star}

The \model STaR algorithm (\cref{alg:model-star}) takes as input $(C, \{ x_i^*, y^*_i \mid i \in [1 .. \mathcal{D}_{\text{train}}] \})$, where $C$ is the \model\,to train, and $\{ (x_i^*, y^*_i) \mid i \in [1 .. \mathcal{D}_{\text{train}}] \}$ is the training dataset. As we described in \cref{sec:tac-star}, \model STaR uses a `fallback \model ' heuristics in hope to obtain a sample when the forward algorithm fails.

\paragraph{Building Fallback \model.} Given a \model\,$C = (\mathbf{Z}, \mathbf{E})$ with input node and output node typed $\tau_i$ and $\tau_o$ respectively, we build its fallback \model\, ${C}_{\text{fallback}} = (\mathbf{Z}', \mathbf{E}')$ (denoted as the function \texttt{build\_fallback} in \cref{alg:model-star}) as follows:
\begin{itemize}
    \item The input node of ${C}_{\text{fallback}}$: $\mathbf{z}'_1$ is of the product type $\tau_{io} = \tau_i \times \tau_o$, representing a data container that holds one object of type $\tau_i$ and another object of type $\tau_o$.
    \item All other nodes $\in \mathbf{Z}$ have their counterpart nodes in $Z'$ (with the same types and indices).
    \item We copy each hyperedge $e \in \mathbf{E}$ over to $\mathbf{E}'$, connecting nodes with the same indices. In the case that $e$ is a deterministic algorithm hyperedge, and has $\mathbf{z}_1$ as one of its source nodes, we modify the counterpart hyperedge $e'$ to have a deterministic algorithm that first extracts the original object $\texttt{parse}(\mathbf{z}_1)$ from $\texttt{parse}(\mathbf{z}'_1)$, and then pass $\texttt{parse}(\mathbf{z}_1)$ to the original algorithm as input.
\end{itemize}
Adaptors in ${C}_{\text{fallback}}$ use no-op weights, falling back to the behavior of the base model. We denote such no-op weights as $\vtheta_0$. For example, \cref{fig:fallback-tac} is the ${C}_{\text{fallback}}$ for \cref{fig:cot-type-structure}.

\begin{algorithm}[h]
\caption{\model STaR Training Algorithm}
\label{alg:model-star}
\begin{algorithmic}[1]
\Require Training pairs $\mathcal{D}_{\text{train}} = \{ (x^*_i, y^*_i) \mid i \in [1 .. |\mathcal{D}_{\text{train}}|] \}$, \model\, $C$, sampler configuration $\kappa$.
\State ${C}_{\text{fallback}} \leftarrow \text{build\_fallback}(C)$
\For{epoch in $[1 .. \texttt{num\_epochs}]$}
\State $S \leftarrow \{ \}$ {\bf \# Successful samples}
\For{training pair $(x^*, y^*) \in \mathcal{D}_{\text{train}}$}
    \State $\mathbf{z}^*_1 \leftarrow \texttt{canon}(x^*)$
    \State \textbf{\# E-step (Sampling Latent Variables):}
    \State $({\mathbf{\hat{z}}}_2 \ldots {\mathbf{\hat{z}}}_M) \leftarrow \text{Forward}(C, \mathbf{z}^*_1)$.
    \State \textbf{\# Filtering (Validity Check):}
    \State Initialize \texttt{error\_flag} $\leftarrow$ false.
    \State Set \texttt{error\_flag} $\leftarrow$ true if errors in E-step or $\texttt{parse}({\mathbf{\hat{z}}}_2)  \neq y^*$.
    \State \textbf{\# Heuristics Fallback (Addressing Forward Failure):}
    \If{\texttt{error\_flag} is true}
        \State $\mathbf{z}'^*_1 \leftarrow \texttt{canon}((x^*, y^*))$
            \State $({\mathbf{\hat{z}}}'_2 \ldots {\mathbf{\hat{z}}}'_M) \leftarrow \texttt{forward}({C}_{\text{fallback}}, \mathbf{z}'^*_1)[0]$.
            \If{no error was raised and $\texttt{parse}({\mathbf{\hat{z}}}'_2) = y^*$}
                \State $({\mathbf{\hat{z}}}_2 \ldots {\mathbf{\hat{z}}}_M) \leftarrow ({\mathbf{\hat{z}}}'_2 \ldots {\mathbf{\hat{z}}}'_M)$
                \State Set \texttt{error\_flag} $\leftarrow$ false.
            \EndIf
    \EndIf
    \If{\texttt{error\_flag} is false}
        \State $S \leftarrow S \cup \{ (\mathbf{z}^*_1, {\mathbf{\hat{z}}}_2 \ldots {\mathbf{\hat{z}}}_M) \}$
    \EndIf
\EndFor
\State \textbf{\# M-step (Parameter Update):}
\For{$(\mathbf{z}^*_1, {\mathbf{\hat{z}}}_2 \ldots {\mathbf{\hat{z}}}_M) \in S$}
    \State $\mathbf{G} \leftarrow \texttt{backward}(C, (\mathbf{z}^*_1, {\mathbf{\hat{z}}}_2 \ldots {\mathbf{\hat{z}}}_M))[2]$
    \State $\texttt{optimize}(C, \mathbf{G})$
\EndFor
\EndFor
\end{algorithmic}
\end{algorithm}

\subsection{Amortized \model STaR}

The Amortized \model STaR algorithm (\cref{alg:amortized-model-star}) builds upon \cref{alg:model-star} to introduce an inference network \model . While ${C}_{\text{fallback}}$ used fixed no-op weights that behave identical to the base language model, Amortized \model STaR leverages an inference network \model\,$C'$ with trainable parameters.

\paragraph{Building the inference network $C'$.} Given a \model\,$C = (\mathbf{Z}, \mathbf{E})$ with input node and output node typed $\tau_i$ and $\tau_o$ respectively, we build the adaptive fallback \model ${C}' = (\mathbf{Z}', \mathbf{E}')$ (denoted as the function \texttt{build\_infer\_net} in \cref{alg:amortized-model-star}). At a high level, every adaptor hyperedge that generates latent variables in $C$ is mapped into a counterpart in $C'$ that also depends on both observed a $\tau_i$-typed input and a $\tau_o$-typed output, now encoded as $\mathbf{z}'_1$, typed $\tau_{io}$. Specifically we build $C'$ with the following procedure:
\begin{itemize}
    \item The input node of ${C}'$: $\mathbf{z}'_1$ is of the product type $\tau_{io} = \tau_i \times \tau_o$, as with \texttt{build\_fallback}.
    \item All nodes $\in \mathbf{Z}$ have their counterpart nodes in $Z'$ (with the same types and indices), except for $\{ \mathbf{z}_1, \mathbf{z}_2 \}$.\footnote{We arbitrarily designate a node $\in \mathbf{Z}'$ that does not have an outgoing hyperedge as the output node for syntactic conformity.}
    \item For each hyperedge $e \in \mathbf{E}$,
    \begin{itemize}
        \item In the case that $e$ is a deterministic algorithm hyperedge, and has $\mathbf{z}_1$ as one of its source nodes, we add a counterpart hyperedge $e'$ that connect counterpart nodes in $\mathbf{Z}'$, with its deterministic algorithm modified to typecheck, as \texttt{build\_fallback}. 
        \item Otherwise, $e$ is an adaptor hyperedge. Denoting its source node as $\mathbf{z}_s$ and target node as $\mathbf{z}_t$:
    \begin{itemize}
        \item If $\mathbf{z}_t = \mathbf{z}_2$, we continue since $\mathbf{z}_t$ has no counterpart $C'$.
        \item If $\mathbf{z}_s = \mathbf{z}_1$ and $\mathbf{z}_t \neq \mathbf{z}_2$, we add a counterpart hyperedge $e' = (\tau_s, \tau_t, \vtheta_{\text{new}})$ connecting counterpart nodes $\mathbf{z}'_s$ and $\mathbf{z}'_t$. $\vtheta_{\text{new}}$ indicates the parameter vector of a new LM adaptor.
        \item Otherwise, $\mathbf{z}_s \neq \mathbf{z}_1$ and $\mathbf{z}_t \neq \mathbf{z}_2$. In this case, we create $e'$ to be an adaptor that is conditioned on both $\mathbf{z}'_s$ and $\mathbf{z}'_1$. To achieve this goal, we introduce into $C'$ a helper node $\mathbf{z}''_s$ typed $\tau_{ios} = \tau_i \times \tau_o \times \tau_s $, and a helper hyperedge $e''$ that has source nodes $\{ \mathbf{z}'_1, \mathbf{z}'_s \}$, and target node $\{ \mathbf{z}''_s \}$. $e''$ is a deterministic edge that combines values in $\mathbf{z}'_1$ and $\mathbf{z}'_s$ into the $3$-object container $\mathbf{z}''_s$. Finally, we add $e'$ that connects $\mathbf{z}''_s$ to $\mathbf{t}$ as the adaptor transformation $(\tau_{ios}, \tau_t, \vtheta_{\text{new}})$, where  $\vtheta_{\text{new}}$ again indicates the parameter vector of a new LM adaptor.
    \end{itemize}
    \end{itemize}
\end{itemize}

Adaptors in $C'$ are new adaptors. And we train $C$ alternately with $C'$ in \cref{alg:amortized-model-star}. The algorithm to train $C'$ is listed in \cref{alg:update-amortized}.

\begin{algorithm}[h]
\caption{Amortized \model STaR Training Algorithm}
\label{alg:amortized-model-star}
\begin{algorithmic}[1]
\Require Training pairs $\mathcal{D}_{\text{train}} = \{ (x^*_i, y^*_i) \mid i \in [1 .. |\mathcal{D}_{\text{train}}|] \}$, \model\, $C$, sampler configuration $\kappa$.
\State ${C}' \leftarrow \text{build\_infer\_net}(C)$
\For{epoch in $[1 .. \texttt{num\_epochs}]$}
\State $S \leftarrow \{ \}$ {\bf \# Successful samples}
\For{training pair $(x^*, y^*) \in \mathcal{D}_{\text{train}}$}
    \State $\mathbf{z}^*_1 \leftarrow \texttt{canon}(x^*)$
    \State \textbf{\# E-step (Sampling Latent Variables):}
    \State $({\mathbf{\hat{z}}}_2 \ldots {\mathbf{\hat{z}}}_M) \leftarrow \text{Forward}(C, \mathbf{z}^*_1)$.
    \State \textbf{\# Filtering (Validity Check):}
    \State Initialize \texttt{error\_flag} $\leftarrow$ false.
    \State Set \texttt{error\_flag} $\leftarrow$ true if errors in E-step or $\texttt{parse}({\mathbf{\hat{z}}}_2)  \neq y^*$.
    \State \textbf{\# Heuristics Fallback (Addressing Forward Failure):}
    \If{\texttt{error\_flag} is true}
        \State $\mathbf{z}'^*_1 \leftarrow \texttt{canon}((x^*, y^*))$
            \State $({\mathbf{\hat{z}}}'_2 \ldots {\mathbf{\hat{z}}}'_M) \leftarrow \texttt{forward}({C}_{\text{fallback}}, \mathbf{z}'^*_1)[0]$.
            \If{no error was raised and $\texttt{parse}({\mathbf{\hat{z}}}'_2) = y^*$}
                \State $({\mathbf{\hat{z}}}_2 \ldots {\mathbf{\hat{z}}}_M) \leftarrow ({\mathbf{\hat{z}}}'_2 \ldots {\mathbf{\hat{z}}}'_M)$
                \State Set \texttt{error\_flag} $\leftarrow$ false.
            \EndIf
    \EndIf
    \If{\texttt{error\_flag} is true}
        \State $({\mathbf{\hat{z}}}_3 \ldots {\mathbf{\hat{z}}}_M) \leftarrow \texttt{forward}(C', \mathbf{z}^*_1)[0]$
        \State Set \texttt{error\_flag} $\leftarrow$ false if no errors in previous step.
    \EndIf
    \If{\texttt{error\_flag} is false}
        \State $S \leftarrow S \cup \{ (\mathbf{z}^*_1, \mathbf{z}^*_2, \mathbf{\hat{z}}_3, \ldots {\mathbf{\hat{z}}}_M) \}$
    \EndIf
\EndFor
\State \textbf{\# M-step (Parameter Update):}
\For{$(\mathbf{z}^*_1, {\mathbf{\hat{z}}}_2 \ldots {\mathbf{\hat{z}}}_M) \in S$}
    \State $\mathbf{G} \leftarrow \texttt{backward}(C, (\mathbf{z}^*_1, {\mathbf{\hat{z}}}_2 \ldots {\mathbf{\hat{z}}}_M))[2]$
    \State $\texttt{optimize}(C, \mathbf{G})$
\EndFor
\State $C' \leftarrow $ update inference network $C'$ (\cref{sec:update-amortized}).
\EndFor
\end{algorithmic}
\end{algorithm}

\subsection{Updating $C'$}
\label{sec:update-amortized}

We train the inference network $C'$ to better approximate the posterior distribution defined by $C$ alternately (\cref{sec:amortized_model_star}).
In other words, we update adaptor parameters in $C'$ so that sampled latent variables of $C'$ ($(\mathbf{\hat{z}}_3, \ldots, \mathbf{\hat{z}}_M)$ obtained using $\texttt{forward}(C', \texttt{canon}(x^*), \kappa)$) follow the normalized distributions under $C$ (obtained using $\texttt{backward}(C, (\texttt{canon}(x^*), \texttt{canon}(y^*), \mathbf{\hat{z}}_3, \ldots, \mathbf{\hat{z}}_M)))$). To promote diversity of samples, we additionally obtain samples from $C_{\text{fallback}}$ (\cref{sec:appendix-tac-star}). Let $\mathbf{Z} = (\mathbf{z}^*_3, \ldots, \mathbf{z}^*_M)$ be a sample out of $G$ collected samples $(\mathbf{Z}^{(1)}, \ldots, \mathbf{Z}^{(G)})$ from $C_{\text{fallback}}$ and $C'$. We approximate the posterior probability of $\mathbf{Z}$ under $C$, conditioning on $\mathbf{z}^*_1 = \texttt{canon}(x^*)$, $\mathbf{z}^*_2 = \texttt{canon}(y^*)$ under the balance heuristic \citep{veach1995} as
\begin{align}
    \hat{p}_{\text{posterior}}(\mathbf{Z}) \propto \frac{(N_{\text{fallback}} + N_{\text{infer}}) \tilde{p}_{\text{model}}  }{N_{\text{fallback}} p_{\text{fallback}} + N_{\text{infer}} p_{\text{infer}}},
    \label{eq:balance-heuristic}
\end{align}
where $\tilde{p}_{\text{model}} = \tilde{p}_{C}(\mathbf{z}^*_1, \mathbf{z}^*_2, \mathbf{z}^*_3, \ldots, \mathbf{z}^*_M)$, $p_{\text{fallback}} = \prod_{m=3}^{M} p_{LM}(\mathbf{z}^*_m \mid \mathbf{z}^*_m \text{'s source node}; \vtheta_{0})$, and $p_{\text{infer}} = \prod_{m=3}^{M} p_{LM}(\mathbf{z}^*_m \mid \mathbf{z}^*_m \text{'s source node}; \vtheta_{\text{new}})$. These values are all obtained using the $\texttt{backward}$ algorithm.\footnote{\texttt{backward} algorithm as presented in this work computes both gradients and probabilities. In our implementation we do not compute gradients when they are not needed; but we omit this subtlety in \cref{alg:cascade-backward}.}
We denote the number of samples attempted (including errors) on $C_{\text{fallback}} = N_{\text{fallback}}$, the number of samples attempted (including errors) on $C' = N_{\text{infer}}$. $\hat{p}_{\text{posterior}}$ is normalized over the  mixture so that $\sum_{g = 1}^{G} \hat{p}_{\text{posterior}}(\mathbf{Z}^{(g)}) = 1$.

\Cref{alg:update-amortized} updates adaptors in $C'$ to bring its unnormalized distribution closer to \cref{eq:balance-heuristic}. Since the self-normalized approximation of the posterior distribution is consistent but biased, we require minimum numbers of samples from $C'$ and $C_{\text{fallback}}$.

\begin{algorithm*}[h]
\caption{\texttt{update\_infer\_net}}
\label{alg:update-amortized}
\begin{algorithmic}[1]
\Require Training pair $(x^*, y^*)$, model \model\, $C$, sampler configuration $\kappa$, inference network $C'$, non-adaptive fallback $C_{\text{fallback}}$, number of samples from $C_{\text{fallback}}$: $G_{\text{fallback}}$, number of samples from $C'$: $G_{\text{infer}}$ .
\State $\mathbf{z'}^*_1 \leftarrow \texttt{canon}((x^*, y^*))$, $\mathbf{z}^*_1 \leftarrow \texttt{canon}(x^*)$, $\mathbf{z}^*_2 \leftarrow \texttt{canon}(y^*)$.
\State $\mathbf{Z}_{\text{collected}} \leftarrow []$
\State {\bf \# In our implementation we give up and raise an error after $30$ unsuccessful attempts.}
\While{number of successful samples from $C_{\text{fallback}} < G_{\text{fallback}}$}
\State Try $(\mathbf{\hat{z}}_2, \mathbf{\hat{z}}_3, \ldots \mathbf{\hat{z}}_M) \leftarrow \texttt{forward}(C_{\text{fallback}}, \mathbf{z'}^*_1, \kappa, 1)$
\If{previous step succeeded}
\State {\bf \# We discard $\mathbf{\hat{z}}_2$ from $C_{\text{fallback}}$.}
\State Append $(\mathbf{\hat{z}}_3, \ldots, \mathbf{\hat{z}}_M)$ to $\mathbf{Z}_{\text{collected}}$.
\EndIf
\EndWhile
\State $N_{\text{fallback}} \leftarrow $ numbers of attempts on $C_{\text{fallback}}$
\While{number of successful samples from $C' < G_{\text{infer}}$}
\State Try $(\mathbf{\hat{z}}_3, \ldots \mathbf{\hat{z}}_M) \leftarrow \texttt{forward}(C', \mathbf{z'}^*_1, \kappa, 1)$
\If{previous step succeeded}
\State Append $(\mathbf{\hat{z}}_3, \ldots, \mathbf{\hat{z}}_M)$ to $\mathbf{Z}_{\text{collected}}$.
\EndIf
\EndWhile
\State $N_{\text{infer}} \leftarrow $ numbers of attempts on $C'$
\State $G \leftarrow G_{\text{fallback}} + G_{\text{infer}}$
\State Assert $G = |\mathbf{Z}_{\text{collected}}|$
\State Compute $[\hat{p}_{\text{posterior}}(\mathbf{Z}^{(1)} \ldots \hat{p}_{\text{posterior}}(\mathbf{Z}^{(G)})]$ using \cref{eq:balance-heuristic}.
\State Sample $g \in [1 .. G]$ with probability proportional to $\hat{p}_{\text{posterior}}(\mathbf{Z}^{(g)}$.
\State $\mathbf{G} \leftarrow \texttt{backward}(C', \mathbf{Z}^{(g)})[2]$.
\State $\texttt{optimize}(C', \mathbf{G})$
\end{algorithmic}
\end{algorithm*}

\section{Formal Statements and Proofs Regarding Type Compliance}

\label{sec:partition-function-discussion}

\paragraph{Well-specifiedness.} Let $C = (\mathbf{Z}, \mathbf{E})$.
We define well-specifiedness for TAC: we say $\vtheta = \{ \vtheta_1 \ldots \vtheta_K \} $ is well-specified if for every LM adaptor $e_k = (\tau_i, \tau_o, \vtheta_k) \in \mathbf{E}$ and for every valid canonicalized string $\vx$ of type $\tau_i$, the LM distribution $p_{LM}$ only has support over valid outputs of type $\tau_o$. Formally, $\forall \text{ valid }x, \sum_{\vy \in \mathrm{\valid}(\tau_o)} p_{LM}(\vy \mid \vx ; \vtheta_k) = 1$ iff $\vtheta$ is well-specified.

We first prove that hyperedges are locally normalized (\emph{i.e.}, the partition function is $1$) when $\vtheta$ is well-specified:

\begin{lemma}
If $\vtheta$ is well-specified, then for any hyperedge $e_k \in \mathbf{E}$ and any valid assignment $\vx$ to its source nodes, the local partition function $Z_k = 1$.
\label{thm:local-well-specified}
\end{lemma}
\begin{proof}
    $e_k$ is either an LM adaptor or a deterministic algorithm:
\begin{itemize}
    \item If $e_k$ is an LM adaptor, $Z_k = \sum_{\vy} \tilde p_{\vtheta} (\vy \mid \vx; e_k) = \sum_{\vy \in \mathrm{valid}(\tau_o) } p_{LM}(\vy \mid \vx ; \vtheta_k) = 1$.
    \item If $e_k$ is a deterministic algorithm, by \cref{eq:deterministic_prob} $Z_k = \sum_{\vy} \tilde{p}(\vy \mid \vx; e_k) = \tilde{p}(\texttt{canon}(f(\texttt{parse}(\vx, \tau_i))) + 0 = 1 + 0 = 1$.
\end{itemize}
\end{proof}

We then use induction based on the \model\,$C$'s topological structure.

\begin{lemma}
Let $\vtheta$ be a well-specified parameter vector for \model\,$C = (\mathbf{Z}, \mathbf{E})$. The conditional partition function $\mathcal{Z}_{\vtheta}(\mathbf{z}_1) = 1$.
\label{thm:induction}
\end{lemma}
\begin{proof}
    We use induction on the number of nodes $k$, following the topological sort $\mathbf{z}_1, \ldots \mathbf{z}_M$. For clarity, here we abuse the subscript notation for topological order, and therefore $\mathbf{z}_M$ (instead of $\mathbf{z}_2$) is the output.

Let $C_k$ be the sub-\model\,induced by $\{\mathbf{z}_1, \ldots , \mathbf{z}_k \}$. Its partition function is $\mathcal{Z}_k(\mathbf{z}_1 ) = \sum_{\mathbf{z}_2 \ldots \mathbf{z}_k} \prod_{m=2}^{k} \tilde{p}_{\vtheta}(\mathbf{z}_m \mid S_m),$ where $S_m$ denotes the source nodes of $\mathbf{z}_m$ under its corresponding hyperedge.

\paragraph{Base Case.} $k=1$. $C_1$ has only $\mathbf{z}_1$. $\mathcal{Z}_1(\mathbf{z}_1) = 1$ since the product is empty.

\paragraph{Inductive Step.} We assume $\mathcal{Z}_{k-1}(\mathbf{z}_1) = 1$. First we rewrite $\mathcal{Z}_k(\mathbf{z}_1)$ by explicitly summing over $\mathbf{z}_k$. Since $\mathbf{z}_1, \ldots \mathbf{z}_k$ is a topological order, the source nodes of $\mathbf{z}_k$: $S_k$ is a subset of $\{ \mathbf{z}_1, \ldots \mathbf{z}_{k-1}\}$. We thus rewrite $\mathcal{Z}_k(\mathbf{z}_1)$ as
\begin{align}
\mathcal{Z}_k(\mathbf{z}_1) &= \sum_{ \mathbf{z}_2 \ldots \mathbf{z}_k } \left( \prod_{m=2}^{k-1} \tilde{p}_{\theta}(\mathbf{z}_m \mid S_m) \right) \cdot \left(\sum_{\mathbf{z}_k} \tilde{p}_{\theta}( \mathbf{z}_k \mid S_k ) \right).
\label{eq:z-k-rewrite}
\end{align}

We discuss the summands by the validity of $\mathbf{z}_2 \ldots \mathbf{z}_{k-1}$:
\begin{itemize}
\item If $\mathbf{z}_2 \ldots \mathbf{z}_{k-1}$ is valid: by \cref{thm:local-well-specified} the term $\sum_{\mathbf{z}_k} \tilde{p}_{\vtheta}( \mathbf{z}_k \mid S_k ) = 1$. This summand is therefore $\prod_{m=2}^{k-1} \tilde{p}_{\vtheta}(\mathbf{z}_m \mid S_m)$.
\item If $\mathbf{z}_2 \ldots \mathbf{z}_{k-1}$ is not valid: by Eqs 1 and 2 this summand is 0.
\end{itemize}

We can thus rewrite \cref{eq:z-k-rewrite} as 
\begin{align}
\mathcal{Z}_k(\mathbf{z}_1) &= \sum_{\mathbf{z}_2, \ldots, \mathbf{z}_{k-1} \mid \text{valid assignments} } \prod_{m=2}^{k-1} \tilde{p}_{\vtheta}(\mathbf{z}_m \mid S_m).
\label{eq:z-k-rewrite-2}
\end{align}

\Cref{eq:z-k-rewrite-2} can be further rewritten to sum over both valid and invalid $\mathbf{z}_2, \ldots , \mathbf{z}_{k-1}$ assignments (since again by \cref{eq:adaptor_prob,eq:deterministic_prob}, the summand is $0$ for invalid assignments):
\begin{align}
    \mathcal{Z}_k(\mathbf{z}_1) &= \sum_{\mathbf{z}_2, \ldots, \mathbf{z}_{k-1} } \prod_{m=2}^{k-1} \tilde{p}_{\vtheta}(\mathbf{z}_m \mid S_m) = \mathcal{Z}_{k-1}(\mathbf{z}_1).
\end{align}

Since by assumption $\mathcal{Z}_{k-1}(\mathbf{z}_1) = 1$, we thus prove by induction $\mathcal{Z}_M(\mathbf{z}_{1}) = \mathcal{Z}_{\vtheta}(\mathbf{z}_1) = 1$.
\end{proof}

Finally, we show that \cref{thm:induction} implies the equivalence of maximizing the normalized and unnormalized likelihoods when the true parameters are well-specified.
\unnormalizedisgoodenough*
\begin{proof}
First we note $\forall \vtheta \in \vTheta , \mathcal{Z}_{\vtheta}(\mathbf{z}_1) \leq 1 $, since for any adaptor $\sum_{y} \tilde{p}_{\vtheta} ( \vy \mid \vx) \leq 1 $ . By \cref{eq:adaptor_prob,eq:deterministic_prob} the global partition function must also be $\leq 1$.

We rewrite the unnormalized likelihood as a product of normalized likelihood and the partition function:
\begin{align}
\tilde{p}_{\vtheta}(\mathbf{z}_{2 \ldots M} \mid \mathbf{z}_1) &= p_{\vtheta}(\mathbf{z}_{2 \ldots M} \mid \mathbf{z}_1) \cdot \mathcal{Z}_{\vtheta} ( \mathbf{z}_1 )
\end{align}

Since $\mathcal{Z}_{\vtheta}(\mathbf{z}_1) \leq 1$, $\forall \vtheta \in \vTheta , \tilde{p}_{\vtheta}(\mathbf{z}_{2 \ldots M} \mid \mathbf{z}_1) \leq p_{\vtheta}(\mathbf{z}_{2 \ldots M} \mid \mathbf{z}_1)$.

At the well-specified true parameters $\vtheta^*$ we have $\mathcal{Z}_{\vtheta}(\mathbf{z}_1) = 1$ by \cref{thm:induction}. Therefore $\tilde{p}_{\vtheta^*}(\mathbf{z}_{2 \ldots M} \mid \mathbf{z}_1) = p_{\vtheta^*}(\mathbf{z}_{2 \ldots M} \mid \mathbf{z}_1)$.

By our assumption that $\vtheta^*$ maximizes normalized likelihood, $\forall \vtheta \in \vTheta , p_{\vtheta^*}(\mathbf{z}_{2 \ldots M} \mid \mathbf{z}_1) \geq p_{\vtheta}(\mathbf{z}_{2 \ldots M} \mid \mathbf{z}_1)$.

Combining everything together:
\begin{align*}
\tilde{p}_{\vtheta^*}(\mathbf{z}_{2 \ldots M} \mid \mathbf{z}_1) &= p_{\vtheta^*}(\mathbf{z}_{2 \ldots M} \mid \mathbf{z}_1) \\ & \geq p_{\vtheta}(\mathbf{z}_{2 \ldots M} \mid \mathbf{z}_1) \\ & \geq \tilde{p}_{\vtheta}(\mathbf{z}_{2 \ldots M} \mid \mathbf{z}_1)
\end{align*}
for all $\vtheta \in \vTheta $. Under the assumption $\vtheta^*$ is unique, $\vtheta^* = \argmax_{\vtheta \in \vTheta} \tilde{p}_{\vtheta}(\mathbf{z}_{2 \ldots M} \mid \mathbf{z}_1) = \hat{\vtheta}$.
\end{proof}

\boundingmlegradients*
\begin{proof}
Here we fix $\mathbf{z}_1 = x$. We denote $\mathbf{z}_{2 \ldots M} = y$.
Let $p_{LM}^{(k)}(y)$ be the $k$-th LM adaptor's unmasked node probability, given $(x, y)$ as \model\,input and output. We then denote
$p_{\vtheta}(y) = \prod_{k} p_{LM}^{(k)}$ as a \model 's \emph{normalized} distribution over node assignments (without masking invalid ones).
    The partition function $\mathcal{Z}_{\vtheta} = \sum_{y} p_{\vtheta}(y \mid x) \mathbb{I}(y \in V) = \mathrm{Pr}_{p_{\vtheta}}(V)$ where $V$ is the set of valid node assignments.

We first rewrite $\nabla_{\vtheta} \log \mathcal{Z}_{\vtheta}$ as an expectation under $p_{\vtheta}$:
\begin{align}
    \nabla_{\vtheta} \log \mathcal{Z}_{\vtheta} &= \mathbb{E}_{y \sim p_{\vtheta}( \cdot \mid V)}\left[ \nabla_{\vtheta} \log p_{\vtheta}(y) \right].
    \label{eq:grad-log-z}
\end{align}
Using the identity $\sum_{y} p_{\vtheta}(y) \nabla_{\vtheta} \log p_{\vtheta}(y) = 0$, we rewrite \cref{eq:grad-log-z} as
\begin{align}
    \nabla_{\vtheta} \log \mathcal{Z}_{\vtheta} &= \mathbb{E}_{y \sim p_{\vtheta}( \cdot \mid V)}\left[ \nabla_{\vtheta} \log p_{\vtheta}(y) \right] - \mathbb{E}_{y \sim p_{\vtheta}}\left[ \nabla_{\vtheta} \log p_{\vtheta}(y) \right].
    \label{eq:grad-log-z-2}
\end{align}
Let $f = \nabla_{\vtheta} \log p_{\vtheta}(y)$. We can now rewrite $\Vert \nabla_{\vtheta} \log \mathcal{Z}_{\vtheta} \Vert_{\infty}$ as
\begin{align}
    \Vert \nabla_{\vtheta} \log \mathcal{Z}_{\vtheta} \Vert_{\infty} &= \Vert \mathbb{E}_{p_{\cdot \mid V}} \left[ f \right] - \mathbb{E}_{p_{\vtheta}} \left[ f \right] \Vert_{\infty} \nonumber \\
    &= \Vert \sum_{y} f \cdot ( p_{\vtheta}(y \mid V) - p_{\vtheta}(y) )   \Vert_{\infty} \nonumber \\
    & \leq \sum_{y} \Vert f \Vert_{\infty} \cdot | p_{\vtheta}(y \mid V) - p_{\vtheta}(y) | \nonumber\\
    & \leq \sum_{y} G \cdot  | p_{\vtheta}(y \mid V) - p_{\vtheta}(y) |.
    \label{eq:grad-log-z-magnitude}
\end{align}
Noting that $\sum_{y} | p_{\vtheta}(y \mid V) - p_{\vtheta}(y) |$ is twice the total variation between $p_{\vtheta}$ and $p_{\vtheta}(\cdot \mid V)$, and that the total variation between $p_{\vtheta}$ and $p_{\vtheta}(\cdot \mid V)$ is $(1 - \mathcal{Z}_{\vtheta})$ \Dash the sum of invalid assignments' probabilities under $p_{\vtheta}$ \Dash we can rewrite \cref{eq:grad-log-z-magnitude} as $\Vert \nabla_{\vtheta} \log \mathcal{Z}_{\vtheta} \Vert_{\infty} \leq 2 G (1 - \mathcal{Z}_{\vtheta})$.

\end{proof}

\section{Implementation Considerations}

\label{sec:implementation}

In this section we discuss practical implementation considerations. In particular, we distinguish between \emph{one-time} and \emph{per-use} efforts.

\subsection{One-Time Efforts}

\paragraph{Parsing and canonicalization.} There exist multiple libraries that can readily be used to implement \texttt{parse} and \texttt{canon} for typed data-holding objects in Python. One example is LangFun which we use extensively in the paper. Another popular library is Pydantic, which is used in DSPy.

\paragraph{Type validation logic.} As we briefly discussed in \cref{ft:validation-logic}, the \texttt{parse} function can be used to implement complex business logic. Such logic can usually be implemented cleanly as part of type definition (\emph{e.g.}, as \texttt{\_\_init\_\_} and \texttt{\_\_post\_init\_\_} methods in Python).

\paragraph{Algorithms.} The core \model\,algorithms for execution and training (Algorithms listed in \cref{sec:algorithms}) are general and need only be implemented once. The main computational bottlenecks in these algorithms are:
\begin{itemize}
    \item Sampling from an LM adaptor $p_{LM}(\cdot ; \vtheta)$.
    \item Evaluating the conditional probability of $\vy$ given $\vx$ under an LM adaptor: $p_{LM}(\vy \mid \vx ; \vtheta)$.
    \item Computing gradients of $(\vx, \vy)$ with regard to parameters $\vtheta$: $\nabla_{\vtheta} \log p_{LM}(\vy \mid \vx ; \vtheta)$.
\end{itemize}
A practical implementation can abstract these bottlenecks away, by offloading these intensive parts to dedicated inference servers (e.g., vLLM). The core \model\,logic remains a lightweight, accelerator-agnostic program. Furthermore, since \model s use parameter-efficient fine-tuning (PEFT), the adaptor weights and gradients are small enough to be processed quickly, often without needing dedicated accelerators for the logic itself. This design significantly reduces the low-level engineering burden.

\subsection{Per-Use Efforts}

Once the core engine is in place, a practitioner's effort is focused on defining a \model\,hypergraph for their specific task.
Since the \model\,hypergraph is essentially a data flow graph, it can be represented in a way that is directly analogous to network architecture definitions in popular neural network frameworks such as PyTorch, where the \texttt{Module} s represent hyperedges, and their \texttt{forward} methods connect the typed data nodes.

\section{Additional \model\,Diagrams of Trainable Workflows}

\label{sec:model-figures}

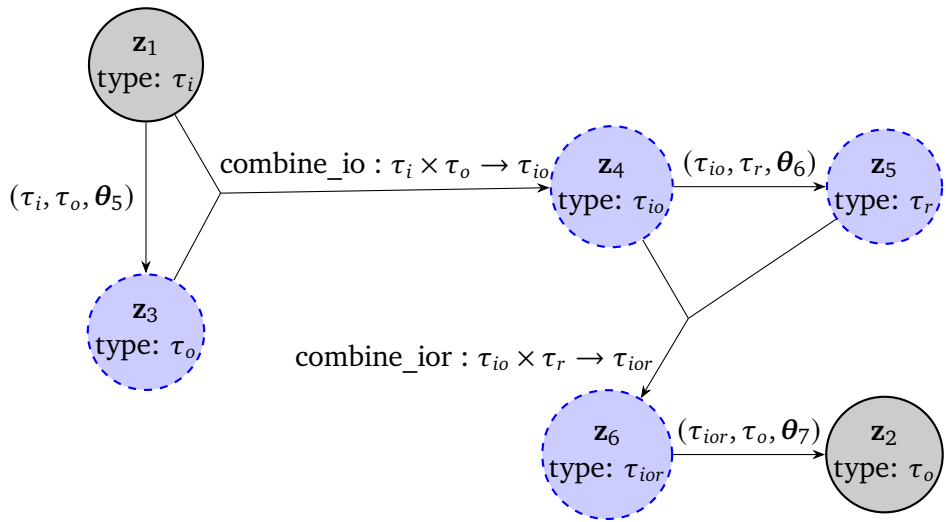
\begin{figure*}[h]
\centering
\begin{tikzpicture}[>=Stealth]

    \node[observed, align=center] (x) {$\mathbf{z}_1$ \\ type: $\tau_i$};
    \node[latent, align=center, below=2cm of x] (yprime) {$\mathbf{z}_3$ \\ type: $\tau_o$};
    \node[latent, align=center,  above right=.8cm and 5cm of yprime] (iyprime) {$\mathbf{z}_4$ \\ type: $\tau_{io}$};
    \node[latent, align=center, right=2cm of iyprime] (r) {$\mathbf{z}_5$ \\ type: $\tau_r$};
    \node[observed, align=center, below=2cm of r] (y) {$\mathbf{z}_2$ \\ type: $\tau_o$};
    \node[latent, align=center, left=2cm of y] (iyprimer) {$\mathbf{z}_6$ \\ type: $\tau_{ior}$};
    \draw[->] (x) edge node[left] {$(\tau_i, \tau_{o}, \vtheta_5)$}  (yprime);
    \hyperedge[-60]{x}{yprime}{iyprime}{$\mathrm{combine\_io}: \tau_i \times \tau_o \rightarrow \tau_{io}$};
    \hyperedgeleft[-60]{iyprime}{r}{iyprimer}{$\mathrm{combine\_ior}: \tau_{io} \times \tau_r \rightarrow \tau_{ior}$};
    \draw[->] (iyprime) edge node[above] {$(\tau_{io}, \tau_{r}, \vtheta_6)$}  (r);
    \draw[->] (iyprimer) edge node[above] {$(\tau_{ior}, \tau_{o}, \vtheta_7)$}  (y);
\end{tikzpicture}
\caption{{\bf refine-structure}: refinement through cascade topology engineering. This cascade models a refinement process where an initial output sketch is iteratively refined based on generated rationales.}
\label{fig:multiple-adaptor-tac-refinement}
\end{figure*}

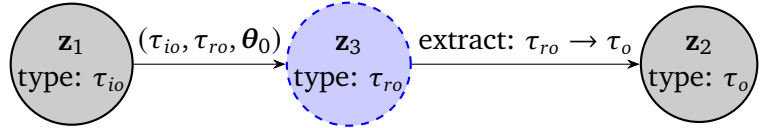
\begin{figure*}[h]
\centering
\begin{tikzpicture}[>=Stealth]

    \node[observed, align=center] (x) {$\mathbf{z}_1$ \\ type: $\tau_{io}$};
    \node[latent, align=center, right=2cm of x] (oprime) {$\mathbf{z}_3$ \\ type: $\tau_{ro}$};
    \node[observed, align=center, right=3cm of oprime] (y) {$\mathbf{z}_2$ \\ type: $\tau_o$};

    \draw[->] (x) edge node[above] {$(\tau_{io}, \tau_{ro}, \vtheta_0)$}  (oprime);
    \draw[->] (oprime) edge node[above] {$\mathrm{extract}$: $\tau_{ro} \rightarrow \tau_o$} (y);

\end{tikzpicture}
\caption{$C_{\text{fallback}}$ for {\bf cot-type-structure}. Notice that the adaptor $(\tau_{io}, \tau_{ro}, \vtheta_0)$ uses `fallback' weights $\vtheta_0$ that represent no-op weights. Since we conduct experiment on LoRA adaptors in this work, we use the zero-init vectors as $\vtheta_0$.}
\label{fig:fallback-tac}
\end{figure*}

\begin{figure*}[h]
\centering
\begin{tikzpicture}[>=Stealth]
    \node[observed, align=center] (x) {$\mathbf{z}_1$ \\ type: $\tau_{io}$};
    \node[latent, align=center, below=2cm of x] (r) {$\mathbf{z}_3$ \\ type: $\tau_r$};
    \node[latent, align=center,  above right=.8cm and 5cm of r] (ir) {$\mathbf{z}_2$ \\ type: $\tau_{ir}$};
    \draw[->] (x) edge node[left] {$(\tau_{io}, \tau_{r}, \vtheta_8)$}  (r);
    \hyperedge[-60]{x}{r}{ir}{$\mathrm{combine\_io\_ir}: \tau_{io} \times \tau_{r} \rightarrow \tau_{ir}$};
\end{tikzpicture}
\caption{Inference network \model\,$C'$ for {\bf cot-type-structure}.}
\label{fig:inference-net-tac}
\end{figure*}
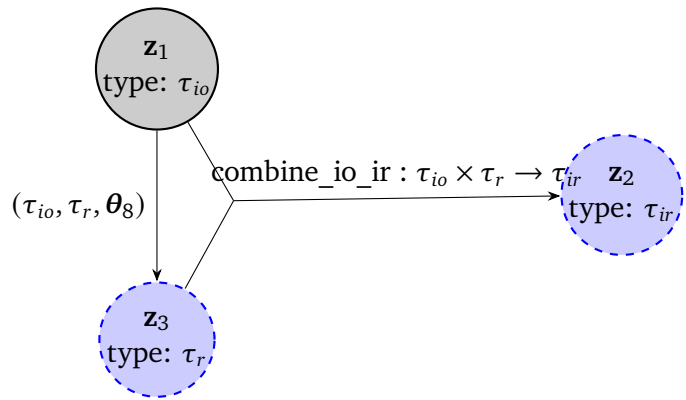

\FloatBarrier

\section{Further Details of Experiment Setup}

\label{sec:setup-appendix}

\paragraph{Data splits.} We focus on the low-data regime of task adaptation in this work. For MGSM and MGSM-SymPy, each language has $100$/$30$/$120$ training/validation/test examples respectively. The splits are $100$/$30$/$100$ and $100$/$30$/$300$ for HotPotQA and FinQA respectively. For HotPotQA and FinQA, we use the first entries from the original dataset files as our training and evaluation subsets. For MGSM experiments, we train and evaluate on each language separately. For MuSR tasks, the splits are $100/30/120$ and $100/30/126$ respectively.

\paragraph{Evaluation.} We look at exact match accuracy scores of the answers for all $5$ tasks.
For MGSM-SymPy experiments, we convert answers from the dataset to integers; as for the model predictions, we evaluate the expressions as rational numbers under SymPy\footnote{\url{https://www.sympy.org/en/index.html}}, and cast the results as integer numbers. We do not make use of additional clues from the datasets (\emph{e.g.}, the rationales provided for the $8$ examples in MGSM datasets).

\subsection{\model\,setup}

\paragraph{Training procedure.} We train all workflows that have latent variables with our \model STaR and Amortized \model STaR algorithms, except for the original (untyped) STaR experiments. Since {\bf direct} experiments do not have latent variables, we train those models using the ordinary cross entropy loss.  In all experiments we use a batch size of $8$. The Adam optimizer \citep{kingma2014adam} is used throughout all experiments, with a learning rate of $5e-5$. We early-stop if no higher validation score is achieved for $4$ consecutive epochs. The sampler configuration $\kappa$ is set to use a combination of top-K and nucleus sampling \citep{nucleussampling}, where we first choose the top $40$ candidates, and cut off accumulated probability mass at $0.95$. To train the inference \model s, we accumulate $32$ samples from $C_{\text{infer}}$ and $16$ samples from the fallback model (that is, $G = 48$ at the end of \cref{alg:update-amortized}).

\paragraph{Decoding procedure for generation tasks.} Here we denote the answer type as $\tau_o$. For each test input instance, we obtain $32$ samples $\mathbf{\hat{Z}}^{(1)} \ldots \mathbf{\hat{Z}}^{(32)}$ using \texttt{forward}, bucket their output node values $\texttt{parse}(\mathbf{\hat{z}}_2^{(1)}, \tau_o) \ldots \texttt{parse}(\mathbf{\hat{z}}_2^{(32)}, \tau_o)$ into $B$ bins, identified by the parsed output $y_1 \ldots y_B$. We output the answer with maximum accumulated unnormalized probability mass, namely $\argmax_{b} \sum_{ s \in [1 .. 32], \mathrm{parse}(\mathbf{\hat{z}}_2, \tau_o) = y_b} \tilde{p}_{\vtheta}(\mathbf{\hat{Z}}^{(s)})$.

\paragraph{Decoding procedure for classification tasks.} We estimate each label $c$'s normalized marginal probability using \cref{eq:unnormalized-marginal-estimate}, with $N=32$. We output the label with largest normalized marginal probability as prediction.

\paragraph{Object representation of data.} We represent input $\tau_i$ and output $\tau_o$ as Python types. The objects are encoded as string representations under LangFun. We design the input and output types separately to reflect the original dataset schemata (\cref{lst:mgsm,lst:hotpotqa,lst:finqa}). As for the rationales (represented by $\tau_r$ in {\bf cot-type-structure} and {\bf cot-cascade-structure}) we represent them as lists of strings (\cref{lst:rationale}). Product types are represented as new Python classes (\emph{e.g.}, the product of type \texttt{Question} and \texttt{Answer}, represented as $\tau_{io}$ in \cref{fig:fallback-tac,fig:inference-net-tac}, is a new class \text{QuestionAnswer}). The object representation can be arbitrarily complex, with LangFun handling all \texttt{canon} and \texttt{parse} logic (for example, \cref{lst:hotpotqarefine} has \texttt{Answer} objects embedded in multiple types; and \cref{lst:mgsmexpression} has self-referential definitions).

\begin{lstlisting}[language=Python, caption={Input and output type definitions for MGSM}, label=lst:mgsm]
class Question:
  question: str


class Answer:
  answer: str
\end{lstlisting}

\begin{lstlisting}[language=Python, caption={Input and output type definitions for HotPotQA}, label=lst:hotpotqa]
class Paragraph:
  title: str
  sentences: list[str]


class Context:
  paragraphs: list[Paragraph]



class Answer:
  answer: str


class Question:
  id: str
  question: str
  context: Context
\end{lstlisting}

\begin{lstlisting}[language=Python, caption={Input and output type definitions for FinQA}, label=lst:finqa]
class Question:
  question: str
  pre_text: list[str]
  table: list[list[str]]
  post_text: list[str]


class Step:
  op: str
  arg1: str
  arg2: str
  res: str


class Answer:
  answer: str


class QuestionAnswer:
  question: Question
  answer: Answer


class Answer:
  answer: str
\end{lstlisting}

\begin{lstlisting}[language=Python, caption={\texttt{Rationale} type definition}, label=lst:rationale]
class Rationale:
  steps: list[str]
\end{lstlisting}

\begin{lstlisting}[language=Python, caption={\texttt{QuestionAnswer} type definition}, label=lst:questionanswer]
class QuestionAnswer:
  question: Question
  answer: Answer
\end{lstlisting}

\begin{lstlisting}[language=Python, caption={Type definitions for {\bf refine-structure} on HotPotQA}, label=lst:hotpotqarefine]
class ThinkingSteps:
  steps: list[str]


class Paragraph:
  title: str
  sentences: list[str]


class Context:
  paragraphs: list[Paragraph]


class SupportingFact:
  title: str
  sentence: str


class RelevantContext:
  sentences: list[str]


class Answer:
  answer: str


class Question:
  id: str
  question: str
  context: Context


class QuestionAnswer:
  question: Question
  answer: Answer


class AnswerFirstAttemptThinkingStepsAnswer:
  answer_first_attempt: Answer
  thinking_steps: ThinkingSteps
  answer: Answer


class QuestionAnswerFirstAttempt:
  question: Question
  answer_first_attempt: Answer


class QuestionAnswerFirstAttemptThinkingSteps:
  question: Question
  answer_first_attempt: Answer
  thinking_steps: ThinkingSteps
\end{lstlisting}

\begin{lstlisting}[language=Python, caption={\texttt{Expression} type definitions in MGSM {\bf expression-cascade-structure} experiments}, label=lst:mgsmexpression]
class Expression:
  operator: Literal['+', '-', '*', '/']
  left: Union[int, 'Expression']
  right: Union[int, 'Expression']

class Answer:
  answer: Expression
\end{lstlisting}

\subsection{DSPy setup}

\label{sec:dspy-setup}

We conduct most of the DSPy experiments under v 3.0.1, but report results from DSPy v 2.6.19 for \texttt{gemini-1.1-7b-it} experiments since both BFSWRS and MIPROv2 struggle to generate valid outputs under DSPy v 3.0.1. Moreover, the non-optimized MGSM average accuracy is much lower under v 3.0.1 (for Native CoT it is $0.7\%$ under v 2.6.19, and $0.2\%$ under v 3.0.1).
For all other experiments, we report results from DSPy v 3.0.1 which sets up JSON schema-based constrained decoding correctly out-of-the-box. As we noted in \cref{sec:effectiveness-against-prompt-optimization}, constrained decoding significantly improves performance for tasks with structured output.

We serve base models on vLLM v 0.10.0.

\paragraph{Input and output object definitions.} For structured input and output tasks, we subclass \texttt{dspy.Signature} as \texttt{QASignature} to represent examples. The property names and types in a \texttt{QASignature} class are identical to counterparts in \model\,experiments. FinQA and MGSM-SymPy signatures are listed in \cref{lst:dspy-signature-finqa} and \cref{lst:dspy-signature-mgsmexpression} respectively.

\begin{lstlisting}[language=Python, caption={DSPy object signature for FinQA. Property names and types are identical to their \model\,counterparts in \cref{lst:finqa}}, label=lst:dspy-signature-finqa]
class QASignature(dspy.Signature):
  pre_text: list[str] = dspy.InputField()
  table: list[list[str]] = dspy.InputField()
  post_text: list[str] = dspy.InputField()
  question: str = dspy.InputField()
  answer: str = dspy.OutputField()
\end{lstlisting}

\begin{lstlisting}[language=Python, caption={DSPy object signature for MGSM-SymPy. Property names and types are identical to their \model\,counterparts in \cref{lst:mgsmexpression}}, label=lst:dspy-signature-mgsmexpression]
class Expression(pydantic.BaseModel):
  operator: Literal['+', '-', '*', '/']
  left: Union[int, float, 'Expression']
  right: Union[int, float, 'Expression']

class QASignature(dspy.Signature):
  question: str = dspy.InputField()
  answer: Expression = dspy.OutputField()
\end{lstlisting}

\paragraph{DSPy models.} We conduct reasoning experiments on both the native \texttt{dspy.ChainOfThought} module, and an explicitly two-step composite module that resembles \model\,{\bf cot-cascade-structure} patterns. Two-step modules for FinQA and MuSR are listed in \cref{lst:dspy-two-step-finqa,lst:dspy-two-step-musr} as examples.

\begin{lstlisting}[language=Python, caption={DSPy two-step reasoning model definition for FinQA}, label=lst:dspy-two-step-finqa]
class QuestionRationale(dspy.Signature):
  question: str = dspy.InputField()
  pre_text: list[str] = dspy.InputField()
  table: list[list[str]] = dspy.InputField()
  post_text: list[str] = dspy.InputField()
  question: str = dspy.InputField()
  rationale: list[str] = dspy.OutputField()

class RationaleAnswer(dspy.Signature):
  rationale: list[str] = dspy.InputField()
  answer: str = dspy.OutputField()

class TwoStepPredictor(dspy.Module):
  def __init__(self):
    self.question_to_rationale = dspy.Predict(QuestionRationale)
    self.rationale_to_answer = dspy.Predict(RationaleAnswer)

  def forward(self, pre_text: list[str], table: list[list[str]], post_text: list[str], question: str):
    r = self.question_to_rationale(question=question, pre_text=pre_text, table=table, post_text=post_text).rationale
    return dspy.Prediction(answer=self.rationale_to_answer(rationale=r).answer)
\end{lstlisting}

\begin{lstlisting}[language=Python, caption={DSPy two-step reasoning model definition for MuSR}, label=lst:dspy-two-step-musr]
class QuestionRationale(dspy.Signature):
  context: str = dspy.InputField()
  question: str = dspy.InputField()
  choices: list[str] = dspy.InputField()
  rationale: list[str] = dspy.OutputField()

class RationaleAnswer(dspy.Signature):
  rationale: list[str] = dspy.InputField()
  choices: list[str] = dspy.InputField()
  answer: str = dspy.OutputField()

class TwoStepPredictor(dspy.Module):
  def __init__(self):
    self.question_to_rationale = dspy.Predict(QuestionRationale)
    self.rationale_to_answer = dspy.Predict(RationaleAnswer)

  def forward(self, context: str, question: str, choices: list[str]):
    r = self.question_to_rationale(question=question, context=context, choices=choices).rationale
    return dspy.Prediction(answer=self.rationale_to_answer(rationale=r, choices=choices).answer)
\end{lstlisting}

\paragraph{Prompt optimization under DSPy.} We experiment with optimizers \texttt{dspy.MIPROv2} and \texttt{dspy.BootstrapFewShotWithRandomSearch} (listed as BFSWRS below). For MGSM-SymPy and FinQA experiments we do not report BFSWRS results, as they consistently need more context length than the model maximum ($8192$). Moreover, for FinQA experiments we resort to MIPROv2 0-shot due to similar context length problems. 

We set \texttt{max\_errors=2} for all optimizers. For {MiPROv2} we set \texttt{auto='medium'}. For MiPROv2 with 0-shot settings we additionally set \texttt{max\_bootstrapped\_demos=0, max\_labed\_demos=0}.

\clearpage

\section{Per-Language \model\, and Original STaR MGSM and MGSM-SymPy results}
\label{sec:per-language-mgsm-results}

Per-language \model\,and original STaR experimental results on tasks MGSM and MGSM-SymPy are listed in \cref{tab:mgsm-numbers,tab:mgsm-numbers-7b}.

\begin{table}[h!]
\pgfplotstableread{
pattern opt es en de fr zh ru ja te th Average
{direct} {\model STaR} 0.275 0.275 0.25 0.25 0.2333333333 0.2583333333 0.2333333333 0.1833333333 0.2666666667 0.2472222222
{cot-type-structure} {\model STaR} 0.8	0.8416666667	0.7666666667	0.8333333333	0.8	0.85	0.7166666667	0.7916666667	0.8333333333	0.8037037037
{cot-cascade-structure} {\model STaR} 0.875	0.875	0.8333333333	0.8583333333	0.8	0.875	0.7416666667	0.7333333333	0.8083333333	0.8222222222
{refine-structure} {\model STaR} 0.8666666667	0.9	0.7666666667	0.775	0.7333333333	0.7833333333	0.6916666667	0.725	0.8333333333	0.7861111111
{expression-cascade-structure} {\model STaR} 0.8333334	0.82500005	0.8333333333	0.7583333333	0.7	0.7916666667	0.6583333333	0.75	0.7583333333	0.7593750063
{cot-cascade-structure} {un-adapted} 0.425	0.475	0.467	0.425	0.450	0.533	0.317	0.450	0.542	0.454
{cot-type-structure} {un-adapted} 0.775	0.7916666667	0.8083333333	0.7666666667	0.6833333333	0.7916666667	0.6833333333	0.6916666667	0.7333333333	0.7472222222
{expression-cascade-structure} {un-adapted} 0.7666666667	0.7166666667	0.6916666667	0.7083333333	0.6833333333	0.6833333333	0.6333333333	0.7083333333	0.7333333333	0.6947916667
{cot-cascade-structure} {amortized \model STaR} 0.8416666667	0.9166666667	0.8666666667	0.8333333333	0.825	0.8166666667	0.7083333333	0.775	0.8333333333	0.8240740741
{N/A} {original STaR} 0.7416666667	0.7916666667	0.7583333333	0.7583333333	0.7	0.8833333333	0.7416666667	0.7583333333	0.7583333333	0.76875
}\mgsmalllanguages
    \centering
    \resizebox{\textwidth}{!}{
    \pgfplotstabletypeset[columns/es/.style={numeric type,precision=1,zerofill,preproc/expr={100*##1}},
    columns/en/.style={numeric type,precision=1,zerofill,preproc/expr={100*##1}},
    columns/de/.style={numeric type,precision=1,zerofill,preproc/expr={100*##1}},
    columns/fr/.style={numeric type,precision=1,zerofill,preproc/expr={100*##1}},
    columns/zh/.style={numeric type,precision=1,zerofill,preproc/expr={100*##1}},
    columns/ru/.style={numeric type,precision=1,zerofill,preproc/expr={100*##1}},
    columns/ja/.style={numeric type,precision=1,zerofill,preproc/expr={100*##1}},
    columns/te/.style={numeric type,precision=1,zerofill,preproc/expr={100*##1}},
    columns/th/.style={numeric type,precision=1,zerofill,preproc/expr={100*##1}},
    columns/Average/.style={numeric type,precision=1,zerofill,preproc/expr={100*##1}},
    columns/pattern/.style={string type, column name={Pattern}},
    columns/opt/.style={string type, column name={Adaptation Method}},
        every head row/.style={before row=\toprule, after row={\midrule}},
        every last row/.style={after row=\bottomrule}
    ]{\mgsmalllanguages}}
\caption{\texttt{gemma-2-27b-it} MGSM and MGSM-SymPy per-language accuracies (\model\,and original STaR experiments).}
\label{tab:mgsm-numbers}
\end{table}

\begin{table}[h!]
\pgfplotstableread{
pattern opt es en de fr zh ru ja te th Average
{direct} {\model STaR} 0.058	0.067	0.067	0.083	0.075	0.025	0.050	0.017	0.017	0.051
{cot-cascade-structure} {\model STaR} 0.408	0.358	0.317	0.292	0.242	0.317	0.133	0.183	0.208	0.273
{cot-cascade-structure} {un-adapted} 0.008	0.000	0.008	0.000	0.000	0.008	0.000	0.017	0.000	0.005
{N/A} {original STaR} 0.150	0.275	0.017	0.058	0.225	0.000	0.033	0.092	0.092	0.105
}\mgsmalllanguagessb
    \centering
    \resizebox{\textwidth}{!}{
    \pgfplotstabletypeset[columns/es/.style={numeric type,precision=1,zerofill,preproc/expr={100*##1}},
    columns/en/.style={numeric type,precision=1,zerofill,preproc/expr={100*##1}},
    columns/de/.style={numeric type,precision=1,zerofill,preproc/expr={100*##1}},
    columns/fr/.style={numeric type,precision=1,zerofill,preproc/expr={100*##1}},
    columns/zh/.style={numeric type,precision=1,zerofill,preproc/expr={100*##1}},
    columns/ru/.style={numeric type,precision=1,zerofill,preproc/expr={100*##1}},
    columns/ja/.style={numeric type,precision=1,zerofill,preproc/expr={100*##1}},
    columns/te/.style={numeric type,precision=1,zerofill,preproc/expr={100*##1}},
    columns/th/.style={numeric type,precision=1,zerofill,preproc/expr={100*##1}},
    columns/Average/.style={numeric type,precision=1,zerofill,preproc/expr={100*##1}},
    columns/pattern/.style={string type, column name={Pattern}},
    columns/opt/.style={string type, column name={Adaptation Method}},
        every head row/.style={before row=\toprule, after row={\midrule}},
        every last row/.style={after row=\bottomrule}
    ]{\mgsmalllanguagessb}}
\caption{\texttt{gemma-1.1-7b-it} MGSM per-language accuracies (\model\,and original STaR experiments).}
\label{tab:mgsm-numbers-7b}
\end{table}

\section{Per-Task \model\, MuSR results}

\label{sec:tac-per-task-musr-results}

Per-task \model\, experimental results on task MuSR are listed in \cref{tab:musr-numbers-tac-gemma-2-27b,tab:musr-numbers-tac-gemma-1.1-7b}.

\begin{table}[h!]
\pgfplotstableread{
DM MM OP TA Average
Generation 0.6166666667	0.5158730159	0.4166666667	0.5164021164
Classification 0.65	0.5	0.8	0.65
}\musrtacallltasks
    \centering
    {
    \pgfplotstabletypeset[columns/MM/.style={numeric type,precision=1,zerofill,preproc/expr={100*##1}, column name={Murder Mystery}},
    columns/OP/.style={numeric type,precision=1,zerofill,preproc/expr={100*##1}, column name={Object Placements}},
    columns/TA/.style={numeric type,precision=1,zerofill,preproc/expr={100*##1}, column name={Team Allocation}},
    columns/Average/.style={numeric type,precision=1,zerofill,preproc/expr={100*##1}},
    columns/DM/.style={string type, column name={Decoding Method}},
        every head row/.style={before row=\toprule, after row={\midrule}},
        every last row/.style={after row=\bottomrule}
    ]{\musrtacallltasks}}
\caption{\texttt{gemma-2-27b-it} MuSR per-task accuracies (\model\, experiments).}
\label{tab:musr-numbers-tac-gemma-2-27b}
\end{table}

\begin{table}[h!]
\pgfplotstableread{
DM MM OP TA Average
Generation 0.6	0.4365079365	0.825	0.6205026455
Classification 0.5916666667	0.4285714286	0.8583333333	0.6261904762
}\musrtacallltaskssb
    \centering
    {
    \pgfplotstabletypeset[columns/MM/.style={numeric type,precision=1,zerofill,preproc/expr={100*##1}, column name={Murder Mystery}},
    columns/OP/.style={numeric type,precision=1,zerofill,preproc/expr={100*##1}, column name={Object Placements}},
    columns/TA/.style={numeric type,precision=1,zerofill,preproc/expr={100*##1}, column name={Team Allocation}},
    columns/Average/.style={numeric type,precision=1,zerofill,preproc/expr={100*##1}},
    columns/DM/.style={string type, column name={Decoding Method}},
        every head row/.style={before row=\toprule, after row={\midrule}},
        every last row/.style={after row=\bottomrule}
    ]{\musrtacallltaskssb}}
\caption{\texttt{gemma-1.1-7b-it} MuSR per-task accuracies (\model\, experiments).}
\label{tab:musr-numbers-tac-gemma-1.1-7b}
\end{table}

\section{Per-Task DSPy MuSR results}

\label{sec:dspy-per-task-musr-results}

Per-task DSPy experimental results on task MuSR are listed in \cref{tab:musr-numbers-dspy-gemma-2-27b,tab:musr-numbers-dspy-gemma-1.1-7b}.

\begin{table}[h!]
\pgfplotstableread{
Model OPT MM OP TA Average
{Native CoT} None 0.2083333333	0	0	0.06944444444
{Native CoT} {MIPRO 0-shot} 0.4083333333	0.007936507937	0	0.1387566138
{Native CoT} {MIPRO} 0.5166666667	0.5079365079	0.4916666667	0.5054232804
{Two-step} None 0.525	0.1428571429	0.225	0.2976190476
{Two-step} {MIPRO 0-shot} 0.55	0.2777777778	0.1916666667	0.3398148148
{Two-step} {MIPRO} 0.5916666667	0.4444444444	0.5083333333	0.5148148148
}\musrdspyallltasks
    \centering
    {
    \pgfplotstabletypeset[columns/MM/.style={numeric type,precision=1,zerofill,preproc/expr={100*##1}, column name={Murder Mystery}},
    columns/OP/.style={numeric type,precision=1,zerofill,preproc/expr={100*##1}, column name={Object Placements}},
    columns/TA/.style={numeric type,precision=1,zerofill,preproc/expr={100*##1}, column name={Team Allocation}},
    columns/Average/.style={numeric type,precision=1,zerofill,preproc/expr={100*##1}},
    columns/OPT/.style={string type, column name={Optimizer}},
        every head row/.style={before row=\toprule, after row={\midrule}},
        every last row/.style={after row=\bottomrule},
    columns/Model/.style={string type, column name={Model}},
    ]{\musrdspyallltasks}}
\caption{\texttt{gemma-2-27b-it} MuSR per-task accuracies (DSPy experiments).}
\label{tab:musr-numbers-dspy-gemma-2-27b}
\end{table}

\begin{table}[h!]
\pgfplotstableread{
Model OPT MM OP TA Average
{Native CoT} None 0.1	0.03174603175	0.03333333333	0.05502645503
{Native CoT} {MIPRO 0-shot} 0.06666666667	0.03174603175	0.025	0.04113756614
{Native CoT} {MIPRO} 0.3416666667	0.253968254	0.5	0.3652116402
{Two-step} None 0.3333333333	0.05555555556	0.1666666667	0.1851851852
{Two-step} {MIPRO 0-shot} 0.3583333333	0.01587301587	0.15	0.1747354497
{Two-step} {MIPRO} 0.4416666667	0.3253968254	0.2666666667	0.3445767196
}\musrdspyallltaskssb
    \centering
    {
    \pgfplotstabletypeset[columns/MM/.style={numeric type,precision=1,zerofill,preproc/expr={100*##1}, column name={Murder Mystery}},
    columns/OP/.style={numeric type,precision=1,zerofill,preproc/expr={100*##1}, column name={Object Placements}},
    columns/TA/.style={numeric type,precision=1,zerofill,preproc/expr={100*##1}, column name={Team Allocation}},
    columns/Average/.style={numeric type,precision=1,zerofill,preproc/expr={100*##1}},
    columns/OPT/.style={string type, column name={Optimizer}},
        every head row/.style={before row=\toprule, after row={\midrule}},
        every last row/.style={after row=\bottomrule},
    columns/Model/.style={string type, column name={Model}},
    ]{\musrdspyallltaskssb}}
\caption{\texttt{gemma-1.1-7b-it} MuSR per-task accuracies (DSPy experiments).}
\label{tab:musr-numbers-dspy-gemma-1.1-7b}
\end{table}

\begin{table}[h!]
\pgfplotstableread{
Model OPT MM OP TA Average
{Native CoT} None 0 0 0 0
{Native CoT} {MIPRO 0-shot} 0 0 0 0
{Native CoT} {MIPRO} 0.5583333333	0.5079365079	0.475	0.5137566138
{Two-step} None 0.04166666667	0.007936507937	0	0.01653439153
{Two-step} {MIPRO 0-shot} 0.03333333333	0.01587301587	0	0.0164021164
{Two-step} {MIPRO} 0.65	0.5952380952	0.6	0.6150793651
}\musrdspyallltasksqwen
    \centering
    {
    \pgfplotstabletypeset[columns/MM/.style={numeric type,precision=1,zerofill,preproc/expr={100*##1}, column name={Murder Mystery}},
    columns/OP/.style={numeric type,precision=1,zerofill,preproc/expr={100*##1}, column name={Object Placements}},
    columns/TA/.style={numeric type,precision=1,zerofill,preproc/expr={100*##1}, column name={Team Allocation}},
    columns/Average/.style={numeric type,precision=1,zerofill,preproc/expr={100*##1}},
    columns/OPT/.style={string type, column name={Optimizer}},
        every head row/.style={before row=\toprule, after row={\midrule}},
        every last row/.style={after row=\bottomrule},
    columns/Model/.style={string type, column name={Model}},
    ]{\musrdspyallltasksqwen}}
\caption{\texttt{Qwen3-8B} MuSR per-task accuracies (DSPy experiments).}
\label{tab:musr-numbers-dspy-qwen-8b}
\end{table}

\section{Per-Language DSPy MGSM and MGSM-SymPy results}

\label{sec:dspy-per-language-mgsm-results}

Per-language DSPy experimental results on tasks MGSM and MGSM-SymPy are listed in \cref{tab:mgsm-numbers-dspy-27b,tab:mgsm-numbers-dspy-7b,tab:mgsm-sympy-numbers-dspy-27b}.

\begin{table}[h!]
\pgfplotstableread{
Model Optimizer es en de fr zh ru ja te th Average
{Native CoT} None 0.550	0.575	0.525	0.517	0.542	0.592	0.450	0.392	0.400	0.505
{Native CoT} {BFSWRS} 0.842	0.892	0.875	0.817	0.750	0.875	0.750	0.775	0.792	0.819
{Native CoT} {MIPROv2} 0.825	0.867	0.817	0.767	0.775	0.842	0.700	0.742	0.758	0.788
{Two-step} None	0.017	0.058	0.025	0.017	0.033	0.017	0.017	0.033	0.050	0.030
{Two-step} {MIPROv2}	0.767	0.833	0.767	0.783	0.733	0.792	0.700	0.675	0.717	0.752
{Two-step} {BFSWRS}	0.808	0.842	0.767	0.817	0.700	0.817	0.675	0.642	0.725	0.755
}\mgsmdspyalllanguages
    \centering
    \resizebox{\textwidth}{!}{
    \pgfplotstabletypeset[columns/es/.style={numeric type,precision=1,zerofill,preproc/expr={100*##1}},
    columns/en/.style={numeric type,precision=1,zerofill,preproc/expr={100*##1}},
    columns/de/.style={numeric type,precision=1,zerofill,preproc/expr={100*##1}},
    columns/fr/.style={numeric type,precision=1,zerofill,preproc/expr={100*##1}},
    columns/zh/.style={numeric type,precision=1,zerofill,preproc/expr={100*##1}},
    columns/ru/.style={numeric type,precision=1,zerofill,preproc/expr={100*##1}},
    columns/ja/.style={numeric type,precision=1,zerofill,preproc/expr={100*##1}},
    columns/te/.style={numeric type,precision=1,zerofill,preproc/expr={100*##1}},
    columns/th/.style={numeric type,precision=1,zerofill,preproc/expr={100*##1}},
    columns/Average/.style={numeric type,precision=1,zerofill,preproc/expr={100*##1}},
    columns/Model/.style={string type, column name={Model}},
    columns/Optimizer/.style={string type, column name={Optimizer}},
        every head row/.style={before row=\toprule, after row={\midrule}},
        every last row/.style={after row=\bottomrule}
    ]{\mgsmdspyalllanguages}}
\caption{\texttt{gemma-2-27b-it} MGSM per-language accuracies (DSPy experiments).}
\label{tab:mgsm-numbers-dspy-27b}
\end{table}

\begin{table}[h!]
\pgfplotstableread{
Model Optimizer es en de fr zh ru ja te th Average
{Native CoT} None	0.008	0.008	0.008	0.000	0.000	0.025	0.017	0.000	0.000	0.007
{Native CoT} {BFSWRS}	0.000	0.008	0.017	0.050	0.008	0.017	0.017	0.025	0.000	0.016
{Native CoT} {MIPROv2}	0.008	0.017	0.025	0.025	0.017	0.000	0.017	0.008	0.008	0.014
{Two-step} None	0.000	0.000	0.008	0.000	0.000	0.000	0.017	0.000	0.000	0.003
{Two-step} {MIPROv2}	0.000	0.000	0.000	0.000	0.000	0.008	0.000	0.000	0.008	0.002
{Two-step} {BFSWRS}	0.000	0.000	0.000	0.000	0.000	0.008	0.000	0.000	0.008	0.002
}\mgsmdspyalllanguagessb
    \centering
    \resizebox{\textwidth}{!}{
    \pgfplotstabletypeset[columns/es/.style={numeric type,precision=1,zerofill,preproc/expr={100*##1}},
    columns/en/.style={numeric type,precision=1,zerofill,preproc/expr={100*##1}},
    columns/de/.style={numeric type,precision=1,zerofill,preproc/expr={100*##1}},
    columns/fr/.style={numeric type,precision=1,zerofill,preproc/expr={100*##1}},
    columns/zh/.style={numeric type,precision=1,zerofill,preproc/expr={100*##1}},
    columns/ru/.style={numeric type,precision=1,zerofill,preproc/expr={100*##1}},
    columns/ja/.style={numeric type,precision=1,zerofill,preproc/expr={100*##1}},
    columns/te/.style={numeric type,precision=1,zerofill,preproc/expr={100*##1}},
    columns/th/.style={numeric type,precision=1,zerofill,preproc/expr={100*##1}},
    columns/Average/.style={numeric type,precision=1,zerofill,preproc/expr={100*##1}},
    columns/Model/.style={string type, column name={Model}},
    columns/Optimizer/.style={string type, column name={Optimizer}},
        every head row/.style={before row=\toprule, after row={\midrule}},
        every last row/.style={after row=\bottomrule}
    ]{\mgsmdspyalllanguagessb}}
\caption{\texttt{gemma-1.1-7b-it} MGSM per-language accuracies (DSPy experiments).}
\label{tab:mgsm-numbers-dspy-7b}
\end{table}

\begin{table}[h!]
\pgfplotstableread{
Model Optimizer es en de fr zh ru ja te th Average
{Native CoT} None	0.567	0.667	0.550	0.458	0.475	0.592	0.450	0.492	0.458	0.523
{Native CoT} {MIPROv2}	0.667	0.642	0.583	0.608	0.567	0.625	0.508	0.425	0.517	0.571
{Two-step} None	0.000	0.000	0.000	0.000	0.000	0.000	0.000	0.000	0.000	0.000
{Two-step} {MIPROv2}	0.000	0.000	0.000	0.000	0.000	0.000	0.000	0.000	0.000	0.000
}\mgsmdspyalllanguagessympy
    \centering
    \resizebox{\textwidth}{!}{
    \pgfplotstabletypeset[columns/es/.style={numeric type,precision=1,zerofill,preproc/expr={100*##1}},
    columns/en/.style={numeric type,precision=1,zerofill,preproc/expr={100*##1}},
    columns/de/.style={numeric type,precision=1,zerofill,preproc/expr={100*##1}},
    columns/fr/.style={numeric type,precision=1,zerofill,preproc/expr={100*##1}},
    columns/zh/.style={numeric type,precision=1,zerofill,preproc/expr={100*##1}},
    columns/ru/.style={numeric type,precision=1,zerofill,preproc/expr={100*##1}},
    columns/ja/.style={numeric type,precision=1,zerofill,preproc/expr={100*##1}},
    columns/te/.style={numeric type,precision=1,zerofill,preproc/expr={100*##1}},
    columns/th/.style={numeric type,precision=1,zerofill,preproc/expr={100*##1}},
    columns/Average/.style={numeric type,precision=1,zerofill,preproc/expr={100*##1}},
    columns/Model/.style={string type, column name={Model}},
    columns/Optimizer/.style={string type, column name={Optimizer}},
        every head row/.style={before row=\toprule, after row={\midrule}},
        every last row/.style={after row=\bottomrule}
    ]{\mgsmdspyalllanguagessympy}}
\caption{\texttt{gemma-2-27b-it} MGSM-SymPy per-language accuracies (DSPy experiments).}
\label{tab:mgsm-sympy-numbers-dspy-27b}
\end{table}

\section{DSPy FinQA results}

\label{sec:dspy-finqa-results}

DSPy experimental results on the FinQA task are listed in \cref{tab:finqa-numbers-dspy-27b} and \cref{tab:finqa-numbers-dspy-7b}.

\begin{table}[h!]
\pgfplotstableread{
Model Optimizer Accuracy
{Native CoT} None	0.1166666667
{Native CoT} {MIPROv2 0-shot}	0.1266666667
{Two-step} None	0.05666666667
{Two-step} {MIPROv2 0-shot} 0.1066666667
}\finqa
    \centering
    {
    \pgfplotstabletypeset[columns/es/.style={numeric type,precision=1,zerofill,preproc/expr={100*##1}},
    columns/Accuracy/.style={numeric type,precision=1,zerofill,preproc/expr={100*##1}},
    columns/Model/.style={string type, column name={Model}},
    columns/Optimizer/.style={string type, column name={Optimizer}},
        every head row/.style={before row=\toprule, after row={\midrule}},
        every last row/.style={after row=\bottomrule}
    ]{\finqa}}
\caption{\texttt{gemma-2-27b-it} FinQA accuracy (DSPy experiments).}
\label{tab:finqa-numbers-dspy-27b}
\end{table}

\begin{table}[h!]
\pgfplotstableread{
Model Optimizer Accuracy
{Native CoT} None	0
{Native CoT} {MIPROv2 0-shot}	0.006666666667
{Two-step} None	0
{Two-step} {MIPROv2 0-shot} 0.003333333333
}\finqasb
    \centering
    {
    \pgfplotstabletypeset[columns/es/.style={numeric type,precision=1,zerofill,preproc/expr={100*##1}},
    columns/Accuracy/.style={numeric type,precision=1,zerofill,preproc/expr={100*##1}},
    columns/Model/.style={string type, column name={Model}},
    columns/Optimizer/.style={string type, column name={Optimizer}},
        every head row/.style={before row=\toprule, after row={\midrule}},
        every last row/.style={after row=\bottomrule}
    ]{\finqasb}}
\caption{\texttt{gemma-1.1-7b-it} FinQA accuracy (DSPy experiments).}
\label{tab:finqa-numbers-dspy-7b}
\end{table}

\begin{table}[h!]
\pgfplotstableread{
Model Optimizer Accuracy
{Native CoT} None	0.04333333333
{Native CoT} {MIPROv2 0-shot}	0.05333333333
{Two-step} None	0.01
{Two-step} {MIPROv2 0-shot} 0.12
}\finqaqwen
    \centering
    {
    \pgfplotstabletypeset[columns/es/.style={numeric type,precision=1,zerofill,preproc/expr={100*##1}},
    columns/Accuracy/.style={numeric type,precision=1,zerofill,preproc/expr={100*##1}},
    columns/Model/.style={string type, column name={Model}},
    columns/Optimizer/.style={string type, column name={Optimizer}},
        every head row/.style={before row=\toprule, after row={\midrule}},
        every last row/.style={after row=\bottomrule}
    ]{\finqaqwen}}
\caption{\texttt{Qwen3-8B} FinQA accuracy (DSPy experiments).}
\label{tab:finqa-numbers-dspy-qwen}
\end{table}

\section{Example Expressions from {\bf expression-cascade-structure} under the MGSM-SymPy task}

See \cref{tab:expression-examples}.

\label{sec:expressions-from-ecs}

\begin{table*}[h!]
\centering
{
\begin{tabular}{p{.7\linewidth}p{.07\linewidth}l}
\toprule
Question   & Answer & Expression      \\
\midrule
{\small Nissa hires 60 seasonal workers to play elves in her department store's Santa village. A third of the elves quit after children vomit on them, then 10 of the remaining elves quit after kids kick their shins. How many elves are left?}	     & 20 & $(60 - (60 / 3)) - 10$ \\
{\small The expenditure of Joseph in May was \$500. In June, his expenditure was \$60 less. How much was his total expenditure for those two months?}	     & 940 & $500 + 440$ \\
{\small Tom gets 4 car washes a month. If each car wash costs \$15 how much does he pay in a year?} & 720 & $(15 \times 4) \times 12$ \\
\bottomrule
\end{tabular}
}
\caption{Example arithmetic expressions generated for MGSM questions by {\bf expression-cascade-structure}.}
\label{tab:expression-examples}
\end{table*}

\section{Example instruction prompt generated by LangFun}

\label{sec:prompts-generated-by-langfun}

The LangFun library translates requests that transformed a typed object into another typed object into natural language instructions for LLMs, to facilitate its \texttt{parse} operations. For example, \cref{lst:example-langfun} is a prompt generated by LangFun for the request that transforms a \texttt{Question} object into an \text{Answer} object.

\begin{lstlisting}[caption={Example instruction prompt generated by LangFun}, label=lst:example-langfun]
Please respond to the last INPUT_OBJECT with OUTPUT_OBJECT according to OUTPUT_TYPE.

INPUT_OBJECT:
  1 + 1 =

OUTPUT_TYPE:
  Answer

  ```python
  class Answer:
    final_answer: int
  ```

OUTPUT_OBJECT:
  ```python
  Answer(
    final_answer=2
  )
  ```

INPUT_OBJECT:
  ```python
  Question(
    question='How are you?'
  )
  ```

OUTPUT_TYPE:
  Answer

  ```python
  class Answer:
    answer: str
  ```

OUTPUT_OBJECT:
\end{lstlisting}

\end{document}